\documentclass[12pt]{article}
\RequirePackage[margin=1in]{geometry}
\RequirePackage[utf8]{inputenc}
\RequirePackage[T1]{fontenc}
\RequirePackage{lmodern}
\RequirePackage{url}
\RequirePackage{amsmath,amssymb,amsthm,mathtools}
\RequirePackage{aliascnt}
\RequirePackage{graphicx}
\RequirePackage{float}
\RequirePackage{booktabs}
\RequirePackage{algorithm}
\RequirePackage{algpseudocode}
\RequirePackage{nicefrac}
\RequirePackage{microtype}
\RequirePackage{xcolor}
\RequirePackage{colortbl}
\RequirePackage{tikz}
\usetikzlibrary{arrows.meta,positioning,fit}
\RequirePackage[round,authoryear]{natbib}
\RequirePackage[colorlinks=true,citecolor=black,linkcolor=black,urlcolor=black]{hyperref}
\newtheorem{theorem}{Theorem}[section]
\newaliascnt{proposition}{theorem}
\newtheorem{proposition}[proposition]{Proposition}
\aliascntresetthe{proposition}
\newaliascnt{lemma}{theorem}
\newtheorem{lemma}[lemma]{Lemma}
\aliascntresetthe{lemma}
\newaliascnt{corollary}{theorem}
\newtheorem{corollary}[corollary]{Corollary}
\aliascntresetthe{corollary}
\newaliascnt{assumption}{theorem}

\aliascntresetthe{assumption}
\theoremstyle{remark}
\newaliascnt{remark}{theorem}
\newtheorem{remark}[remark]{Remark}
\aliascntresetthe{remark}

\RequirePackage[capitalise,noabbrev]{cleveref}
\crefname{theorem}{Theorem}{Theorems}
\Crefname{theorem}{Theorem}{Theorems}
\crefname{proposition}{Proposition}{Propositions}
\Crefname{proposition}{Proposition}{Propositions}
\crefname{lemma}{Lemma}{Lemmas}
\Crefname{lemma}{Lemma}{Lemmas}
\crefname{corollary}{Corollary}{Corollaries}
\Crefname{corollary}{Corollary}{Corollaries}
\crefname{assumption}{Assumption}{Assumptions}
\Crefname{assumption}{Assumption}{Assumptions}
\crefname{remark}{Remark}{Remarks}
\Crefname{remark}{Remark}{Remarks}

\newcommand{\X}{\mathcal X}
\newcommand{\Y}{\mathcal Y}
\newcommand{\A}{\mathcal A}
\newcommand{\R}{\mathbb R}
\newcommand{\E}{\mathbb E}
\renewcommand{\P}{\mathbb P}
\newcommand{\ep}{\varepsilon}

\newcommand{\se}[1]{{\scriptstyle\,\pm\,#1}}
\DeclareMathOperator*{\argmin}{arg\,min}
\DeclareMathOperator*{\argmax}{arg\,max}
\DeclareMathOperator{\Select}{Select}

\definecolor{ed}{RGB}{225,0,0}

\newcommand{\plotorfallback}[2][]{%
  \IfFileExists{#2}{%
    \includegraphics[#1]{#2}%
  }{%
    \fbox{%
      \parbox[c][1.6in][c]{0.85\linewidth}{%
        \centering Missing figure file:\\ \texttt{\detokenize{#2}}%
      }%
    }%
  }%
}

\newcommand{\plotorfallbacktwo}[3][]{%
  \IfFileExists{#2}{%
    \includegraphics[#1]{#2}%
  }{%
    \IfFileExists{#3}{%
      \includegraphics[#1]{#3}%
    }{%
      \fbox{%
        \parbox[c][1.8in][c]{0.85\linewidth}{%
          \centering Missing mean-loss figure file and legacy utility figure file.%
        }%
      }%
    }%
  }%
}

\title{Risk-Controlled Post-Processing of Decision Policies}
\author{
  Sunay Joshi\thanks{Equal Contribution. Correspondence to: \texttt{sunayj@sas.upenn.edu}, 
  \texttt{tawan@wharton.upenn.edu}.}, \quad
  Tao Wang\footnotemark[1], \quad
  Hamed Hassani, \quad 
  Edgar Dobriban \\
  University of Pennsylvania
}
\date{\today}

\begin{document}
\maketitle

\begin{abstract}
Predictive models are often deployed through existing decision policies that stakeholders are reluctant to change unless a risk constraint requires intervention. We study risk-controlled post-processing: given a deterministic baseline policy, choose a new policy that maximizes agreement with the baseline subject to a chance constraint on a user-specified loss. At the population level, we show that the optimal policy has a threshold structure: it follows the baseline except on contexts where switching to the oracle fallback policy yields a large reduction in conditional violation risk. At the finite-sample level, given a fitted fallback policy and score, we develop a post-processing algorithm that uses calibration data to select a threshold. Leveraging tools from algorithmic stability and stochastic processes, we show that under regularity conditions, in the i.i.d.~setting, the expected excess risk of the post-processed policy is $O(\log n/n)$. In the special case when an exact-safe fallback policy is available, the algorithm achieves precise expected risk control under exchangeability. In this setting, we also give high-probability near-optimality guarantees on the post-processed policy. Experiments on a COVID-19 radiograph diagnosis task, an LLM routing problem, and a synthetic multiclass decision task show that targeted post-processing can meet or nearly meet risk budgets while preserving substantially more agreement with the baseline than score-blind random mixing.
\end{abstract}
\section{Introduction}\label{sec:introduction}

Machine learning systems are increasingly used not only to predict outcomes, but also to recommend actions. In many illustrations, however, the deployed action rule is not an unconstrained optimizer. It may be a legacy policy, a clinical workflow, a rule approved by a regulator, or a model whose behavior is already familiar to users. In such settings, it is natural to ask for a \emph{post-processor}: a lightweight wrapper that changes the baseline action only when needed to satisfy a risk constraint.

This paper studies a simple but expressive version of this problem. We are given a deterministic baseline policy $\pi_0:\mathcal X\to\mathcal A$, a base loss $\ell(a,y)$ measuring the consequence of taking action $a$ when the outcome is $y$, a loss cutoff $c$, and a risk budget $\varepsilon$. The loss cutoff defines the violation event $\{\ell(\pi(X),Y)\ge c\}$, and the constraint requires the violation risk $\mathbb P(\ell(\pi(X),Y)\ge c)$ to be at most the risk budget $\varepsilon$. The goal is to construct a policy $\pi$ that agrees with $\pi_0$ as often as possible while satisfying this chance constraint. Thus, the objective encodes fidelity to the existing policy, while the constraint encodes safety, reliability, or domain-specific acceptability.

We begin by demonstrating that the population problem admits an explicit oracle solution: the optimal policy switches from the baseline policy to a fallback policy exactly when the oracle score exceeds a threshold.
Since the oracle fallback policy and scores are unknown, we propose a
finite-sample conformalized post-processing algorithm (\Cref{alg:crc-nonmono-conservative}). After plugging-in for the unknown conditional label distribution, a held-out calibration set is used to select a switching threshold.
In general, the violation loss can be
non-monotone, so we leverage tools from algorithmic stability and stochastic processes to prove an \(O(\log n/n)\) expected excess violation risk guarantee under regularity conditions. 
When an exact-safe fallback policy is available,
monotonicity is restored, yielding precise expected risk control, assuming only data
exchangeability. Experiments on a COVID-19 radiograph diagnosis task, an LLM
thinking-mode routing problem, and a synthetic multiclass task show that the algorithm meets or
nearly meets the desired risk budgets while preserving more agreement with the
baseline than score-blind random mixing. \Cref{fig:workflow} summarizes the post-processing workflow.

\begin{figure}[t]
    \centering
    \resizebox{\textwidth}{!}{%
    \begin{tikzpicture}[
        font=\normalsize,
        node distance=0.65cm and 0.75cm,
        >=Latex,
        box/.style={
            draw=black!70,
            rounded corners=2pt,
            very thick,
            align=center,
            inner sep=5pt,
            minimum height=1.0cm,
            text width=3.0cm
        },
        widebox/.style={
            box,
            text width=3.0cm
        },
        data/.style={
            box,
            fill=gray!10
        },
        base/.style={
            box,
            fill=blue!8,
            draw=blue!55!black
        },
        fallback/.style={
            box,
            text width=3.5cm,
            fill=green!8,
            draw=green!45!black
        },
        calibrate/.style={
            widebox,
            fill=orange!10,
            draw=orange!65!black
        },
        output/.style={
            widebox,
            text width=6.0cm,
            fill=purple!8,
            draw=purple!55!black
        },
        guarantee/.style={
            widebox,
            text width=6.0cm,
            fill=gray!8,
            draw=black!55
        },
        arrow/.style={
            -{Latex},
            very thick,
            draw=black!65
        },
        dep/.style={
            -{Latex},
            thick,
            dashed,
            draw=black!40
        }
    ]
        \node[data] (train) {Training data};
        \node[fallback, right=of train] (fallback) {Fitted fallback $\hat\pi_*$};
        \node[base, above=of fallback] (base) {Baseline $\pi_0$};
        \node[calibrate, below=of fallback] (score) {Fitted score $\hat\Delta$};
        \node[calibrate, right=1.25cm of fallback] (threshold) {Select $\hat\tau$};
        \node[data, below=of threshold] (caldata) {Calibration data\\$(X_i,Y_i): i\in[n]$};
        \node[output, right=of threshold] (switch) {Post-processed policy\\$\displaystyle
            \hat\pi(x;\hat\tau)=\begin{cases}
            \pi_0(x), & \hat\Delta(x)<\hat\tau,\\
            \hat\pi_*(x), & \hat\Delta(x)\ge\hat\tau
            \end{cases}$};
        \node[guarantee, below=of switch] (risk) {Risk control guarantee:\\$\E[L(Z_{n+1}; \hat \tau)] \le \ep + {C\log n}/{n}$}; 

        \draw[arrow] (train) -- (fallback);
        \draw[arrow] (train.south) |- (score.west);
        \draw[dep] (base) -- (fallback);
        \draw[dep] (base.east) to[bend left=24] (score.east);
        \draw[dep] (fallback) -- (score);

        \draw[arrow] (base.east) -- (threshold.west);
        \draw[arrow] (fallback.east) -- (threshold.west);
        \draw[arrow] (score.east) -- (threshold.west);
        \draw[arrow] (caldata) -- (threshold);
        \draw[arrow] (threshold) -- (switch);
    \end{tikzpicture}
    }
    \caption{Workflow for risk-controlled post-processing. A fitted fallback policy and score identify contexts where deviating from the baseline policy is most useful; calibration chooses a threshold, producing a switching policy that satisfies the risk constraint.}
    \label{fig:workflow}
\end{figure}
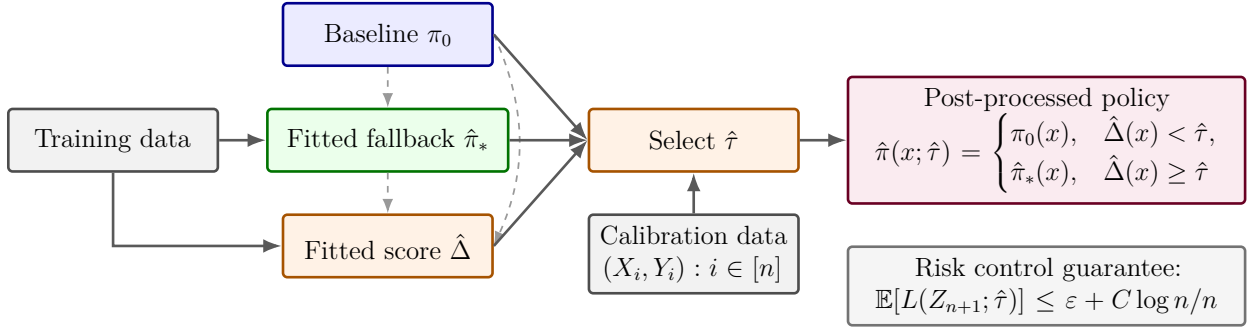

\paragraph{Contributions.}
Our main contributions are as follows.
\begin{itemize}
    \item We formulate risk-controlled policy post-processing as an
    agreement-maximization problem subject to a chance constraint.
    \item We characterize the population-optimal policy, showing that it switches away from the baseline exactly on contexts with the largest oracle scores.
    \item We propose a finite-sample algorithm (\Cref{alg:crc-nonmono-conservative}) to post-process the baseline policy, switching from the baseline policy to a fitted fallback policy only on high-score contexts.
    \item We analyze our algorithm in two regimes. With an
    exact-safe fallback policy, we prove a finite-sample risk control guarantee under exchangeability,
    and show near-optimal agreement with the baseline.
    For a general fitted fallback, the violation loss is non-monotone, and under regularity conditions and i.i.d.~sampling, we bound the excess risk by \(O(\log n/n)\) via rank stability. Our techniques may be of interest in other non-monotone calibration problems.
    \item We evaluate the method on a COVID-19 radiograph diagnosis task, an LLM routing problem, and a synthetic multiclass task, demonstrating that the algorithm meets or nearly meets risk budgets while preserving more agreement with the baseline than score-blind random mixing.
\end{itemize}

\paragraph{Organization.}
\Cref{sec:related-work} discusses related work. \Cref{sec:population-oracle} formulates the problem and derives the population oracle. \Cref{sec:finite-sample} gives the algorithm and theoretical guarantees. \Cref{sec:experiments} presents experiments. \Cref{sec:discussion} contains discussion. Proofs and auxiliary results are deferred to the appendix.

\section{Related Work}\label{sec:related-work}

Conformal prediction provides finite-sample distribution-free predictive guarantees under exchangeability \citep{vovk2005algorithmic}, and conformal risk control extends these ideas to more general user-specified risks \citep{bates2021distribution,angelopoulos2021learn,angelopoulos2022conformal}. One-dimensional conformal risk-control relies on monotonicity of the loss in a tuning parameter,
while recent work has begun to study non-monotone variants via stability \citep{angelopoulos2026conformalriskcontrolnonmonotonic,aldirawi2026conformal}. 
Our analysis is specialized to policy post-processing and exploits the rank stability of the selected threshold.

Our paper also relates to recent work on decision-aware conformal methods, which calibrate uncertainty estimates directly for downstream decision-making \citep{vovk2018conformal,lekeufack2024conformal,cortes2024decision,yeh2024end,patel2024conformal,johnstone2021conformal,kiyani2025decision}. Our focus is complementary: rather than learning a new decision rule from scratch, we seek the least intrusive post-processing of a fixed baseline policy that satisfies a chance constraint.
Due to space limitations, additional related work is reviewed in \Cref{sec:addl-related-work}.

\section{Problem Formulation and Population Oracle}\label{sec:population-oracle}

\subsection{Problem formulation}\label{subsec:problem-formulation}

Let $\mathcal A$ be a finite action space and $\mathcal X$, $\mathcal Y$ be standard Borel spaces.
A \textit{policy} is a measurable deterministic map $\pi:\mathcal X\to\mathcal A$.
Given $X=x$, the deployed action is $\pi(x)$.
Throughout, all policies, scores, and losses are assumed measurable. 

Let $(X,Y)$ have a fixed joint distribution. 
Given a policy $\pi_0:\mathcal X\to\mathcal A$,
a base loss $\ell:\mathcal A\times\mathcal Y\to\mathbb R$, a loss cutoff $c\in\mathbb R$, and a risk budget $\varepsilon\in[0,1]$, define the \textit{violation indicator} as
$\mathbf 1\{\ell(\pi(X),Y)\ge c\}$ and the \textit{violation risk} as
$\mathbb P(\ell(\pi(X),Y)\ge c)$.
The \textit{population-level post-processing problem} is to maximize agreement with the baseline policy subject to controlling violation risk at level $\ep$:
\begin{align}\label{eq:zero-one-obj}
\max_{\pi} \; \mathbb P(\pi(X)=\pi_0(X))
\quad
\text{s.t.}
\quad
\mathbb P\big(\ell(\pi(X),Y)\ge c\big)\le \varepsilon.
\end{align}

\subsection{Population-level optimal policy}\label{subsec:population-optimal-policy}

In this section, we solve the optimization problem in \Cref{eq:zero-one-obj}. We begin with definitions. Define the \textit{conditional violation risk function}
$g : \mathcal A \times \mathcal X \to [0,1]$
by
$g(a,x) := \mathbb P\big(\ell(a,Y)\ge c \mid X = x\big)$ for all $(a,x)\in \mathcal A \times \mathcal X$.
Note that the violation risk can be written as
$\mathbb P\big(\ell(\pi(X),Y)\ge c\big)
= \mathbb E[g(\pi(X),X)]$.
Define the conditional violation risk of the baseline policy $g_0 : \mathcal X \to [0,1]$ as $g_0(x) := g(\pi_0(x), x)$ for all $x\in \mathcal X$.
Define the minimal conditional violation risk $g_* : \mathcal X \to [0,1]$ as
$g_*(x) := \inf_{a\in\mathcal A} g(a,x)$, for all $x \in \mathcal X$. 
Because $\mathcal A$ is finite, after deterministic tie-breaking there exists a measurable selector
$\pi_* : \mathcal X \to \mathcal A$
such that for each $x$,
$\pi_*(x) \in \arg\min_{a\in\mathcal A} g(a,x)$, with the convention that $\pi_*(x)=\pi_0(x)$ whenever
$\pi_0(x)\in\argmin_{a\in\mathcal A}g(a,x)$.
We call $\pi_*$ the \textit{oracle fallback policy.} Define the \textit{oracle score}
$\Delta : \mathcal X \to [0,1]$ by
$\Delta(x) := g_0(x)-g_*(x)$
for $x\in\mathcal X$.
In words, $\Delta(x)$ is the conditional excess violation risk of keeping the baseline action rather than using the oracle fallback policy.

The following result characterizes the population-optimal post-processor. It
shows that the optimal policy has a \textit{threshold structure:} it preserves the
baseline $\pi_0$ on low-score contexts, and switches to the oracle fallback
policy $\pi_*$ on high-score contexts, in such a way that the risk constraint is tight.



\begin{theorem}[Population-level solution to \Cref{eq:zero-one-obj}]\label{thm:zero-one-pop}
Assume that the marginal distribution $P_X$ is atomless.
Let 
$G_*:=\E[g_*(X)]$, $G_0:=\E[g_0(X)]$, and $B:=\varepsilon-G_*$.
If $G_*>\varepsilon$, then \Cref{eq:zero-one-obj} is infeasible. If $G_0\le\varepsilon$, then
$\pi_0$ is optimal. In the non-trivial case $G_*\le\varepsilon<G_0$, define
\[
\tau:=\inf\left\{t\in[0,1]:
\E\!\left[\Delta(X)\mathbf 1\{\Delta(X)<t\}\right]
+t\P(\Delta(X)=t)\ge B
\right\}.
\]
If $B=0$, an optimal deterministic policy is
\[
\pi^*(x)
=
\begin{cases}
\pi_0(x), & \Delta(x)=0,\\
\pi_*(x), & \Delta(x)>0.
\end{cases}
\]
If $B>0$, let $s:=\left(B-\E\!\left[\Delta(X)\mathbf 1\{\Delta(X)<\tau\}\right]\right)/\tau$.
Then $s\in[0,\P(\Delta(X)=\tau)]$, and there exists a measurable set
$E\subseteq\{x:\Delta(x)=\tau\}$ with $P_X(E)=s$. 
An optimal deterministic policy is
\[
\pi^*(x)
=
\begin{cases}
\pi_0(x), & x\in\{\Delta<\tau\}\cup E,\\
\pi_*(x), & x\in\{\Delta\ge\tau\}\setminus E.
\end{cases}
\]
Moreover, for the optimal policy in the non-trivial case, we have $\P\big(\ell(\pi^*(X),Y)\ge c\big)=\varepsilon$ and the agreement with the baseline is
\[
\P(\pi^*(X)=\pi_0(X))
=
\begin{cases}
\mathbb{P}(\Delta(X)=0), & B=0,\\ 
\mathbb{P}(\Delta(X)<\tau)+s, & B>0. 
\end{cases}
\]
\end{theorem}

The proof of the above result is provided in Section \ref{sec:proof-thm-zero-one-pop}.



\section{Finite-Sample Post-Processing Algorithm}\label{sec:finite-sample}

The oracle policy $\pi^*$ defined in \Cref{thm:zero-one-pop} depends on the conditional violation risk function $g$, which is unknown in practice. 
As a result, we use a separate training sample to fit a model to the conditional label distribution, construct a fitted fallback policy and score, and use a held-out calibration sample to choose a threshold.
Specifically, if $\mathcal Y$ is finite and $\hat{f}(x)\in\Delta(\mathcal Y)$ is an estimate\footnote{Throughout, we treat all fitted objects as fixed independently of the calibration and test data.} of the conditional distribution $P_{Y|X=x}$, 
the plug-in estimates of the conditional violation risk $\hat g : \mathcal A \times \mathcal X \to [0,1]$, the fallback policy $\hat \pi_* : \mathcal X \to \mathcal A$, and the score $\hat\Delta : \mathcal X \to [0,1]$ are given by\footnote{Here, ties are broken in favor of $\pi_0(x)$ whenever $\pi_0(x)$ is a minimizer. Also, it is clear that $\hat g, \hat \Delta\in [0,1]$ by definition.}
\begin{equation}\label{eq:plug-in-ests}
\begin{array}{c@{\qquad}c}
\hat g(a,x)
:= \sum_{y\in\mathcal Y}\mathbf 1\{\ell(a,y)\ge c\} \hat{f}_y(x),
&
\hat \pi_*(x)
\in \argmin_{a\in\mathcal A} \hat g(a,x)
\\[0.5em]
\multicolumn{2}{c}{
\hat\Delta(x)
:= \hat g(\pi_0(x),x)-\hat g(\hat \pi_*(x),x)
}
\end{array}
\end{equation}
Next, let $\overline{\mathcal T}:=[0,1]\cup\{\top\}$, where $\top$ is a formal threshold ordered after every element of $[0,1]$, so that $\hat\Delta(x)<\top$ for each $x\in \mathcal X$.
For each $\tau \in \overline{\mathcal T}$, define
the \textit{post-processed policy} $\hat\pi(\cdot;\tau) : \mathcal X \to \mathcal A$
by
\begin{align*}
\hat\pi(x; \tau)
=
\begin{cases}
\pi_0(x), & \hat \Delta(x)< \tau,\\
\hat \pi_*(x), & \hat \Delta(x) \ge \tau,
\end{cases}
\end{align*}
with the convention $\hat\pi(x;\top)=\pi_0(x)$ for each $x\in \mathcal X$.
Finally, define the \textit{violation loss function}
$L : \mathcal X \times \mathcal Y \times \overline{\mathcal T} \to \{0,1\}$
by
$L(x,y;\tau)=\mathbf 1\{\ell(\hat\pi(x;\tau),y)\ge c\}$ for all $x,y,\tau$.

\Cref{alg:crc-nonmono-conservative} is our proposed method for selecting the threshold $\hat \tau$ given $n$ calibration datapoints $Z_i = (X_i, Y_i) : i\in [n]$.\footnote{Here, for a positive integer $n$, $[n] := \{1,\ldots,n\}$.} The algorithm leverages the calibration data to construct a conservative \textit{bumped empirical risk function} $\hat R_n^+(\tau)$, sets $\hat \tau$ to be the largest threshold at which $\hat R_n^+(\tau)$ is controlled at the risk budget $\ep$, and returns the policy $\hat \pi(\cdot; \hat \tau)$. 


\begin{algorithm}
\caption{Post-processing algorithm for plug-in calibration}
\label{alg:crc-nonmono-conservative}

\scalebox{0.9}{%
\begin{minipage}{1.11\linewidth}
\begin{algorithmic}[1]

\State {\bfseries Input:} calibration data $(Z_i)_{i=1}^n$ with $Z_i=(X_i,Y_i)$, fitted score $\hat\Delta:\mathcal X\to[0,1]$, baseline policy $\pi_0$, fitted fallback policy $\hat \pi_*$, risk budget $\varepsilon\in[0,1]$

\State Compute the scores $\hat\Delta_i\gets \hat\Delta(X_i)$ for $i\in[n]$

\State Form the threshold grid
$\mathcal T_n\gets \{0,\top\}\cup\{\hat\Delta_i:i\in[n]\}$, with duplicate values included only once. Sort it in the extended order as
$t_{(0)}<\cdots<t_{(m_n)}$, where $m_n:=|\mathcal T_n|-1$

\For{$j=0,\dots,m_n$}
    \State Define
    \(
    \hat\pi(x;t_{(j)})
    =
    \begin{cases}
    \pi_0(x), & \hat\Delta(x)<t_{(j)},\\
    \hat \pi_*(x), & \hat\Delta(x)\ge t_{(j)}
    \end{cases}
    \)
    \State Compute
    \(
    \hat R_n^+(t_{(j)})
    =
    \frac{1}{n+1}\left(
    \sum_{i=1}^n
    \mathbf 1\{\ell(\hat\pi(X_i;t_{(j)}),Y_i)\ge c\}
    +1
    \right)
    \)
\EndFor

\State Define the feasible set
$\mathcal F_n\gets \left\{ t\in\mathcal T_n:\hat R_n^+(t)\le \varepsilon \right\}$

\State Set \texttt{empty\_feasible} to true iff $\mathcal F_n=\varnothing$

\State If $\mathcal F_n\ne\varnothing$, set
$\hat\tau \gets \max \mathcal F_n$, where the maximum is taken in the extended order; otherwise set $\hat\tau\gets 0$

\State \Return $\hat \pi(\cdot; \hat\tau)$, $\hat\tau$, and \texttt{empty\_feasible}

\end{algorithmic}
\end{minipage}%
}

\end{algorithm}

\subsection{Stability-based risk control guarantee}
\label{subsubsec:crc-nonmono-stability}


\Cref{alg:crc-nonmono-conservative} bears a resemblance to the general Conformal Risk Control (CRC) algorithm \citep{angelopoulos2022conformal}, for which strong distribution-free risk control guarantees are readily available.
However, a crucial assumption behind the CRC guarantee is the monotonicity of the loss function with respect to the calibration threshold.
In our setting, due to the fact that we must estimate the fallback policy $\hat \pi_*$, the violation loss $L(x,y;\tau)$ is not necessarily monotone in $\tau$.
This presents significant technical challenges in the analysis of \Cref{alg:crc-nonmono-conservative}, which we tackle via recently-introduced \textit{stability techniques} in the Conformal Prediction literature \citep{angelopoulos2026conformalriskcontrolnonmonotonic}.

Consider an exchangeable \textit{augmented sample} $Z_i : i\in [n+1]$ consisting of $n$ calibration datapoints and an additional test point $Z_{n+1} = (X_{n+1}, Y_{n+1})$.
For a subset $S\subseteq [n+1]$, we write $Z_S := \{ Z_i : i\in S \}$ for the set of observations with indices in $S$.
We define the \textit{augmented threshold} $\hat \tau_{1:(n+1)}$ as the threshold obtained by running \Cref{alg:crc-nonmono-conservative} on the augmented sample $Z_{[n+1]}$.
Also, given $i\in [n+1]$, we define the \textit{leave-one-out threshold} $\hat \tau_{-i}$ as the threshold obtained by running \Cref{alg:crc-nonmono-conservative} on the leave-one-out sample $Z_{[n+1]\setminus \{i\}}$. (In particular, note that $\hat \tau_{-(n+1)}$ is the threshold we deploy in practice.)

The key quantity in our analysis is the \textit{rank instability} $K$ of our algorithm, which is defined as the maximum displacement between the rank of $\hat \tau_{1:(n+1)}$ and the rank of $\hat \tau_{-i}$ among the scores $\hat \Delta_j : j\in [n+1]$, over $i\in [n+1]$.\footnote{Formal definitions are collected in \Cref{app:formal-rank-stability}.}
The significance of $K$ lies in the fact that only scores $\hat \Delta_j$ whose ranks lie between $\hat \tau_{-i}$ and $\hat \tau_{1:(n+1)}$ can contribute to the loss difference $L(Z_i; \hat \tau_{-i}) - L(Z_i; \hat \tau_{1:(n+1)})$,
as shown in \Cref{prop:stab-K}.
Consequently, if $K$ is small, then by averaging over $i$, taking expectations, and leveraging exchangeability, we deduce that the violation risk $\E[L(Z_{n+1}; \hat \tau_{-(n+1)})]$ of the deployed threshold must lie within $2\E[K]/(n+1)$ of the violation risk $\E[L(Z_{n+1}; \hat \tau_{1:(n+1)})]$ of the augmented threshold.
By the symmetry of $\hat \tau_{1:(n+1)}$ in all $n+1$ datapoints, an additional exchangeability argument then yields risk control up to the rank-instability term and the endpoint-feasibility slack made explicit in \Cref{cor:beta-logn-over-n}.

Therefore, to control the excess violation risk of \Cref{alg:crc-nonmono-conservative}, it suffices to control $\E[K]$. To do so, we leverage techniques from the theory of stochastic processes and, specifically, biased random walks. 
Recall that \Cref{alg:crc-nonmono-conservative} selects a threshold based on the last crossing of a bumped empirical risk function with the level $\ep$.
After ordering observations by their fitted scores, the bumped empirical risk function $\hat R^+(\cdot)$ can be represented as a biased random walk. 
Deleting one observation perturbs this process by at most one, so if the augmented threshold rank and the leave-one-out threshold rank are far apart, the intervening block must have unusually small cumulative drift. Under the drift condition in \Cref{thm:EK-log}, such long low-sum blocks have logarithmic expected length, which implies $\E[K] = O(\log n)$. Formally, we have the following result, which is our crucial theoretical contribution that enables risk control.

\begin{proposition}[Bound on $\E K$]\label{thm:EK-log}
Condition on the fitted objects. Assume the observations $Z_1,\dots,Z_{n+1}$ are i.i.d. and that $\hat\Delta(X)$ has an atomless distribution on $[0,1]$.
Assume $\mathbb{P}(\ell(\hat \pi_*(X_i),Y_i)\ge c) < \varepsilon$ and $\mathbb{P}(\ell(\pi_0(X_i),Y_i)\ge c) > \varepsilon$.
For $i\in [n+1]$, define the random variable
\begin{align*}
W_i = \mathbf 1\{\ell(\pi_0(X_i),Y_i)\ge c\} - \mathbf 1\{\ell(\hat \pi_*(X_i),Y_i)\ge c\}.
\end{align*}
Let $\mu_W:[0,1]\to\mathbb R$ be a version of the conditional mean satisfying
$\E[W_1\mid \hat\Delta(X_1)]=\mu_W(\hat\Delta(X_1))$ a.s., and assume that there exist constants $c_W>0$ and $\beta\ge 0$ such that
$\mu_W(z)\ge c_W z^{\beta}$ 
for $F_{\hat\Delta}$-almost every $z\in (0,1]$, where $F_{\hat\Delta}$ denotes the law of $\hat\Delta(X)$.
Then $\E[K] \le C_1\log(n+1) + C_2$ for constants $C_1, C_2$.\footnote{Here and below, constants do not depend on $n$, but may depend on the data distribution.}
\end{proposition}

The proof of the above result uses the notions in \Cref{app:formal-rank-stability} and is provided in Section \ref{sec:proof-thm-ek-log}.
Combining \Cref{thm:EK-log} and our observations above, we deduce the following violation risk bound.

\begin{theorem}[Risk control from rank stability]\label{cor:beta-logn-over-n}
Under the conditions in \Cref{thm:EK-log}, let $\hat\tau$ be the threshold selected by running \Cref{alg:crc-nonmono-conservative} on the calibration observations $Z_1,\ldots, Z_n$.
Then for all sufficiently large $n$,
$\E\big[L(Z_{n+1};\hat\tau)\big]
\le
\varepsilon
+
{C_3\log(n+1)}/(n+1),
$
for a constant $C_3$.
\end{theorem}

The proof of the above result is provided in Section \ref{sec:proof-cor-beta-logn-over-n}.

\begin{remark}[Interpretation of the drift condition]
The drift condition on $\mu_W$ in \Cref{thm:EK-log} can be rewritten as follows. Let
$\widetilde \Delta(x):=g_0(x)-g(\hat\pi_*(x),x)$
denote the \textit{true} conditional violation risk reduction obtained by switching to the \textit{fitted}
fallback on context $x\in \mathcal X$. By the tower property, since $\E[W\mid X]=\widetilde \Delta(X)$, we have $\mu_W(z)=\E[\widetilde \Delta(X)\mid \hat\Delta(X)=z]$.
Thus, the assumption $\mu_W(z)\ge c_W z^\beta$
implies that contexts assigned fitted score $z$ yield, on average, a
positive benefit from switching, and that this benefit does not vanish too
quickly as $z\downarrow0$.
Further, in the oracle case that $\hat\pi_*=\pi_*$ and $\hat\Delta=\Delta$, we may take $\mu_W(z)=z$, and the condition holds with $c_W=1$ and $\beta=1$.
\end{remark}

\subsection{Risk control guarantee with an exact-safe fallback policy}
\label{subsubsec:safe-policy}


As noted above, for general fallback policies $\hat \pi_*$,
the violation loss $L(Z_i; \tau)$ need not be monotone in
the threshold $\tau$.
However, if the fallback policy $\hat \pi_*$ satisfies a certain \textit{exact-safety} condition, then monotonicity is restored, and \Cref{alg:crc-nonmono-conservative} achieves \textit{distribution-free} risk control at level $\ep$.
Formally, we call a policy $\pi_{\mathrm{safe}} : \mathcal X \to \mathcal A$ \textit{exact-safe} for the base loss $\ell$ and cutoff $c$ if it obeys 
\begin{equation}\label{eq:pointwise-safety}
\ell(\pi_{\mathrm{safe}}(x),y) < c \qquad\text{for all }x\in\mathcal X,\;y\in\mathcal Y.
\end{equation}
Intuitively, $\pi_{\mathrm{safe}}$ is exact-safe if the violation event never occurs.
Note that \(g(\pi_{\mathrm{safe}}(x),x)=0\) for every \(x\in \mathcal X\), so
\(\pi_{\mathrm{safe}}\) is an oracle fallback policy.
Further, exact-safety
also gives \(\hat g(\pi_{\mathrm{safe}}(x),x)=0\), hence the fitted fallback in \Cref{eq:plug-in-ests} may be
taken to be \(\hat \pi_* = \pi_{\mathrm{safe}}\).
The following result shows that under just exchangeability, running \Cref{alg:crc-nonmono-conservative} with an exact-safe fallback controls the violation risk at the desired level.

\begin{theorem}[Conformal risk control with an exact-safe fallback policy]
\label{thm:safe-policy-crc}
If \(Z_1, \ldots, Z_{n+1}\) are exchangeable, if the fitted fallback policy $\hat \pi_*$ is exact-safe, and if \(\hat\tau\) is the threshold selected
by running \Cref{alg:crc-nonmono-conservative} on the calibration set \(Z_1,\dots,Z_n\), then we have
$\E\big[L(Z_{n+1};\hat\tau)\big]\le\varepsilon.$\footnote{If \(\varepsilon<1/(n+1)\), then although the feasible set is empty, \Cref{alg:crc-nonmono-conservative} returns the exact-safe $\hat \pi_*$, and risk control holds.}
\end{theorem}

The proof of the above result is provided in Section
\ref{sec:proof-thm-safe-policy}.


\begin{remark}[Prediction sets as an exact-safe fallback example]
A canonical example arises in prediction set construction. Suppose
\(\mathcal Y\) is finite, \(\mathcal A=2^{\mathcal Y}\), and
\(\ell(a,y)=\mathbf 1\{y\notin a\}\). Taking \(c=1\), the violation risk is the
miscoverage probability. The policy \(\pi_{\mathrm{safe}}(x)=\mathcal Y\) is
exact-safe because \(\ell(\mathcal Y,y)=0\) for every \(y\in\mathcal Y\).
\end{remark}

\subsection{Near-optimality of the finite-sample policy}\label{subsubsec:conc}

Finally, in the exact-safe fallback regime of \Cref{subsubsec:safe-policy}, under mild regularity conditions, we show that the post-processed policy is \textit{nearly optimal} for the problem in \Cref{eq:zero-one-obj}. That is, the post-processor preserves nearly as
much agreement with the baseline as the oracle policy.
The sub-optimality relative to the oracle comes from two sources: finite-sample
calibration noise, and the error incurred by using the fitted score
\(\hat\Delta\) instead of the oracle score \(\Delta\).
This is captured by the following quantity:
\[
\mathcal E_\Delta(u):=
\P(|\hat\Delta(X)-\Delta(X)|>u)
+\sup_{t\in[0,1]}\P(|\Delta(X)-t|\le u).
\]
Denote $(t)_+ := \max\{0, t\}$ for $t\in \R$.

\begin{theorem}[Near-optimality of exact-safe fallback post-processing]
\label{thm:near-opt-short}
Fix \(\delta\in(0,1)\). 
Assume the fitted fallback policy $\hat \pi_*$ is exact-safe, and that the calibration observations are
i.i.d. Further, assume the regularity conditions in \Cref{app:near-optimality-full}.
Finally, define
\(
\ep_\Delta:=\inf_{u>0}\mathcal E_\Delta(u).
\)
Then,
with probability at least \(1-\delta\) over the calibration sample,
if $\ep_{\Delta}$ is sufficiently small,
for sufficiently large \(n\),
\[
\left(
J^*-\P\bigl(\hat\pi(X;\hat\tau)=\pi_0(X)\bigr)
\right)_+
\le
C_4\sqrt{\frac{\log(4/\delta)}{n}}
+
\frac{C_5}{n+1}
+
C_6\varepsilon_\Delta,
\]
where \(J^*\) is the population-optimal objective value for the problem in
\Cref{eq:zero-one-obj},
and the constants $C_4, C_5, C_6 \ge 0$ do not depend on $n$ or $\delta$.\footnote{The constants may depend on the data distribution.}
\end{theorem}

The theorem is stated in its entirety in
\Cref{app:near-optimality-full}, and the proof is provided in \Cref{sec:proof-thm-near-opt}.

\section{Experiments}\label{sec:experiments}
We present two empirical 
illustrations
in the main text, and defer a controlled
synthetic study to Appendix~\ref{app:synthetic_exp}. The empirical illustrations
are a COVID-19 radiograph diagnosis task and an LLM thinking-mode routing task,
which illustrate the method under different fallback structures and cost measures.
Throughout, we distinguish the \textit{exact-safe fallback regime} (where an exact-safe fallback exists, in the sense of \Cref{subsubsec:safe-policy}) from the \textit{non-monotone plug-in regime} (where one does not).
In the diagnosis task, the loss matrix induces
both an exact-safe fallback regime and a non-monotone plug-in regime; in the
routing task, the fallback is a more expensive thinking model and risk control
must be traded against computational cost. In both illustrations, we compare the
proposed score-based post-processor with score-blind random mixing, reporting
violation risk, agreement or switch rate relative to the baseline, and the
relevant performance cost. The synthetic experiment isolates the population
oracle structure and evaluates how closely the proposed algorithm tracks
the oracle agreement--risk tradeoff across exact-safe and non-monotone regimes.




\subsection{Medical diagnosis}
\label{subsec:med_diagnosis}
\paragraph{Setup.}
We follow the decision-theoretic conformal benchmark of \citet{kiyani2025decision} on the COVID-19 Radiography Database \citep{chowdhury2020can,rahman2021exploring}. Each instance is a chest X-ray image labeled with one of four diagnoses: $\Y=\{\text{Normal},\text{Pneumonia},\text{COVID-19},\text{Lung Opacity}\}$, and the decision maker selects one of four clinical actions $\A=\{\text{No Action},\text{Antibiotics},\text{Quarantine},\text{Additional Testing}\}$. 
\citet{kiyani2025decision}  specify the decision through utilities,
while here we work with its loss-form transformation
\(
\ell(a,y):=\max_{(a',y')\in\A\times\Y}u(a',y')-u(a,y)\),
\(\Lambda=
\begin{pmatrix}
0 & 10 & 10 & 9\\
8 & 0 & 7 & 6\\
8 & 7 & 0 & 6\\
6 & 3 & 2 & 0
\end{pmatrix},
\)
with rows of $\Lambda$ indexed by actions, columns by labels, and $\ell(a,y)=\Lambda_{a,y}$. Since $\max_{y\in\Y}\ell(\text{Additional Testing},y)=6$, the action \emph{Additional Testing} is an exact-safe fallback policy whenever $c>6$. No exact-safe fallback policy exists for $c\le 6$.

We choose 
as baseline policy $\pi_0=\mathrm{RAC}(\alpha)$  the Risk-Averse Calibration procedure of \citet{kiyani2025decision}: a
conformal set-valued predictor $\hat C_\alpha:\X\to 2^{\Y}$ at miscoverage level $\alpha$, composed with the max-min decision rule
$
\pi_0(x)\;=\;\argmax_{a\in\A}\;\min_{y\in\hat C_\alpha(x)} u(a,y)
$.
This method has been shown to induce decision-making that is worst-case optimal for the expected $1-\alpha$ quantile of the loss.
As such, it serves as a reasonable heuristic choice for a baseline policy that aims to control violation risk.
However, it is not necessarily guaranteed to control the violation risk, and hence one may ask how to post-process it to ensure this property.

The data are randomly split into training ($70\%$), baseline conformal calibration ($10\%$), threshold calibration ($10\%$), and test ($10\%$) sets. We then fine-tune an Inception-V3 model \citep{szegedy2015going,szegedy2016rethinking} (pretrained on ImageNet)
on the training split, calibrate $\pi_0$ on the baseline calibration split, 
calibrate the threshold $\hat\tau$ on the threshold calibration split, and report all evaluation metrics on the test split. We apply \Cref{alg:crc-nonmono-conservative} to $\pi_0$ in both regimes, with fallback policy taken to be the exact-safe fallback policy $\pi_{\text{safe}}=$\emph{Additional Testing} when $c>6$ and the estimated oracle fallback policy $\hat \pi_*(x)$ when $c\le 6$.
We write $\hat\pi(\cdot;\hat\tau)$ for the post-processed policy returned by the algorithm. 

For any policy $\pi$, we report on the test split the violation risk and mean realized loss.
For the post-processed policy $\hat\pi(\cdot;\hat\tau)$, we also report switch rate relative to $\pi_0$ (precise metric definitions are given in Appendix~\ref{app:additional-covid-results}). All reported quantities below are averages over the $20$ seeds.

\begin{table}
    \centering
    \small
    \setlength{\tabcolsep}{2.7pt}
    \renewcommand{\arraystretch}{1.14}
    \begin{tabular*}{\textwidth}{@{\extracolsep{\fill}}lccccccc@{}}
        \toprule
        \multicolumn{8}{c}{\bfseries Panel A: Exact-safe fallback regime ($c=7$, $\pi_0=\mathrm{RAC}(0.05)$, $\pi_{\mathrm{safe}}\equiv\text{Additional Testing}$)} \\
        \midrule
        & \multicolumn{3}{c}{Violation risk} & \multicolumn{3}{c}{Mean realized loss} & \multicolumn{1}{c}{Switch rate} \\
        \cmidrule(lr){2-4}\cmidrule(lr){5-7}
        $\varepsilon$ & post-processed & $\pi_0$ & $\pi_{\mathrm{safe}}$ & post-processed & $\pi_0$ & $\pi_{\mathrm{safe}}$ & \\
        \midrule
        $0.05$ & $0.04\se{0.00}$ & $0.04\se{0.00}$ & $0.00\se{0.00}$ & \cellcolor{green!30}$0.53\se{0.01}$ & \cellcolor{green!30}$0.53\se{0.01}$ & \cellcolor{red!30}$3.42\se{0.00}$ & $0.00\se{0.00}$ \\
        $0.02$ & $0.02\se{0.00}$ & $0.04\se{0.00}$ & $0.00\se{0.00}$ & \cellcolor{green!15}$0.63\se{0.02}$ & \cellcolor{green!30}$0.53\se{0.01}$ & \cellcolor{red!30}$3.42\se{0.00}$ & $0.06\se{0.01}$ \\
        \bottomrule
    \end{tabular*}

    \vspace{0.4em}

    \begin{tabular*}{\textwidth}{@{\extracolsep{\fill}}lccccccc@{}}
        \toprule
        \multicolumn{8}{c}{\bfseries Panel B: Non-monotone plug-in regime ($\pi_0=\mathrm{RAC}(0.01)$, $\varepsilon=0.1$)} \\
        \midrule
        & \multicolumn{3}{c}{Violation risk} & \multicolumn{3}{c}{Mean realized loss} & \multicolumn{1}{c}{Switch rate} \\
        \cmidrule(lr){2-4}\cmidrule(lr){5-7}
        $c$ & post-processed & $\pi_0$ & $\hat\pi_*$ & post-processed & $\pi_0$ & $\hat\pi_*$ & \\
        \midrule
        $4$ & $0.10\se{0.00}$ & $0.18\se{0.01}$ & $0.06\se{0.00}$ & \cellcolor{green!30}$0.81\se{0.02}$ & \cellcolor{red!30}$1.23\se{0.04}$ & \cellcolor{red!30}$1.01\se{0.01}$ & $0.09\se{0.01}$ \\
        $3$ & $0.10\se{0.00}$ & $0.19\se{0.01}$ & $0.06\se{0.00}$ & \cellcolor{green!30}$0.74\se{0.02}$ & \cellcolor{red!30}$1.23\se{0.04}$ & \cellcolor{green!15}$0.84\se{0.01}$ & $0.11\se{0.01}$ \\
        $2$ & $0.10\se{0.00}$ & $0.24\se{0.01}$ & $0.07\se{0.00}$ & \cellcolor{green!15}$0.61\se{0.01}$ & \cellcolor{red!30}$1.23\se{0.04}$ & \cellcolor{green!30}$0.54\se{0.01}$ & $0.16\se{0.01}$ \\
        \bottomrule
    \end{tabular*}
    \vspace{0.1em}
    \caption{\footnotesize Risk control across exact-safe and non-monotone plug-in fallback regimes. Entries are mean $\pm$ standard error over $20$ random seeds.}
    \label{tab:covid-risk-control}
    \vspace{-1em}
\end{table}

\paragraph{Risk control across fallback regimes.}
To test the exact-safe fallback regime, we take $c=7$, for which \emph{Additional Testing} is an exact-safe fallback policy, and apply \Cref{alg:crc-nonmono-conservative} with $\hat\pi_*=\pi_{\mathrm{safe}}$. Panel A of \Cref{tab:covid-risk-control} reports two configurations at $\pi_0=\mathrm{RAC}(0.05)$. In the first row with risk budget $\varepsilon=0.05$, $\pi_0$ already satisfies the risk constraint ($\pi_0$ violation risk $0.04<\varepsilon$) and the algorithm keeps $\pi_0$ on every test point. In the second row with risk budget $\varepsilon=0.02$, $\pi_0$ is infeasible ($\pi_0$ violation risk $0.04>\varepsilon$); the algorithm switches to $\pi_{\mathrm{safe}}$ at rate $0.06$ and attains the risk-control guarantee (post-processed violation risk $0.02$) with mean realized loss $0.63$, comparable to $\pi_0$'s $0.53$ and far below $\pi_{\mathrm{safe}}$'s $3.42$.

When $c\le 6$, no action is exact-safe; in this non-monotone regime, we use the plug-in estimated oracle fallback $\hat \pi_*(x)\in\argmin_{a\in\A}\hat g(a,x)$. Panel B of \Cref{tab:covid-risk-control} reports results for $\pi_0=\mathrm{RAC}(0.01)$ at $\varepsilon=0.1$ across $c\in\{4,3,2\}$; the rows for $c=5,6$ are omitted because no entry of $\Lambda$ lies in $[4,6)$, so the violation event $\{\ell\ge c\}$ equals $\{\ell\ge 4\}$ for every $c\in\{6,5,4\}$. In Panel B, post-processed violation risk rounds to $0.10$ in every row, consistent with the target $\varepsilon=0.1$ up to the $O(\log n/n)$ stability slack. As $c$ decreases, $\pi_0$ violation risk grows and the switch rate increases, reaching $0.16$ at $c=2$. The algorithm's mean realized loss is strictly below both $\pi_0$'s and $\hat \pi_*$'s for $c\in\{4,3\}$ (e.g., $0.81$ vs.\ $1.23$ and $1.01$ at $c=4$); at $c=2$, where $\hat \pi_*$'s mean realized loss is the smallest of the three policies, the algorithm's mean realized loss ($0.61$) is close to $\hat \pi_*$'s ($0.54$). We also report the corresponding $\pi_0=\mathrm{RAC}(0.02)$ results in \Cref{tab:covid-nonmonotonic-rac002}, where the same qualitative pattern is observed.

\paragraph{Random-mixing comparison.}
To compare the algorithm against the simplest score-blind alternative, we evaluate a random-mixing baseline. For fixed $(\alpha,\varepsilon,c)$, define
\begin{equation}\label{eq:random-mixing}
\pi_{\mathrm{mix},p}(x)=
\begin{cases}
\pi_0(x), & \text{with probability } p,\\
\hat \pi_*(x), & \text{with probability } 1-p,
\end{cases}
\end{equation}
with the Bernoulli draw independent of $(X,Y)$. On the calibration split we set $\hat p_{\mathrm{mix}}$ to the largest value for which the expected mixing violation risk $p\,\hat r_0+(1-p)\,\hat r_*$ does not exceed $\varepsilon$, where $\hat r_0$ and $\hat r_*$ are the empirical violation risks of $\pi_0$ and $\hat \pi_*$ on the calibration sample. With the true endpoint risks in place of $\hat r_0$ and $\hat r_*$, this interpolation would satisfy the risk constraint in expectation; as implemented here, it is a score-blind empirical calibration baseline and does not carry the conformal guarantee of the exact-safe procedure.

\Cref{tab:covid-random-mix} reports both policies for $\pi_0=\mathrm{RAC}(0.01)$ at $\varepsilon=0.1$ and $c\in\{4,3,2\}$. Random mixing and the algorithm both empirically attain violation risk near $\varepsilon$, but the algorithm has a smaller switch rate and a lower mean realized loss in every row. We also report the corresponding $\pi_0=\mathrm{RAC}(0.02)$ comparison in \Cref{tab:covid-random-mix-rac002}, where the same qualitative pattern is observed.

\begin{table}
    \centering
    \small
    \setlength{\tabcolsep}{2.7pt}
    \renewcommand{\arraystretch}{1.14}
    \begin{tabular*}{\textwidth}{@{\extracolsep{\fill}}lcccccc@{}}
        \toprule
        \multicolumn{7}{c}{\bfseries $\pi_0=\mathrm{RAC}(0.01)$, $\varepsilon=0.1$} \\
        \midrule
        & \multicolumn{2}{c}{Violation risk} & \multicolumn{2}{c}{Switch rate} & \multicolumn{2}{c}{Mean realized loss} \\
        \cmidrule(lr){2-3}\cmidrule(lr){4-5}\cmidrule(lr){6-7}
        $c$ & random-mix & post-processed & random-mix & post-processed & random-mix & post-processed \\
        \midrule
        $4$ & $0.10\se{0.00}$ & $0.10\se{0.00}$ & \cellcolor{red!30}$0.22\se{0.01}$ & \cellcolor{green!30}$0.09\se{0.01}$ & \cellcolor{red!30}$1.07\se{0.01}$ & \cellcolor{green!30}$0.81\se{0.02}$ \\
        $3$ & $0.10\se{0.00}$ & $0.10\se{0.00}$ & \cellcolor{red!30}$0.22\se{0.01}$ & \cellcolor{green!30}$0.11\se{0.01}$ & \cellcolor{red!30}$0.95\se{0.01}$ & \cellcolor{green!30}$0.74\se{0.02}$ \\
        $2$ & $0.10\se{0.00}$ & $0.10\se{0.00}$ & \cellcolor{green!15}$0.19\se{0.01}$ & \cellcolor{green!30}$0.16\se{0.01}$ & \cellcolor{green!15}$0.67\se{0.01}$ & \cellcolor{green!30}$0.61\se{0.01}$ \\
        \bottomrule
    \end{tabular*}
    \vspace{0.1em}
    \caption{\footnotesize Our algorithm versus the random-mixing baseline of \eqref{eq:random-mixing} at $\varepsilon=0.1$. Entries are mean $\pm$ standard error over $20$ random seeds.}
    \label{tab:covid-random-mix}
    \vspace{-1em}
\end{table}

\paragraph{Binary diagnostic toy with action cost.}
To allow comparison of policies on action cost separately from the loss, we introduce an explicit action cost; for a clean illustration we use a binary collapse of the COVID experiment. The label and action spaces collapse to
$\Y=\{\text{Normal},\text{Non-Normal}\}, \A=\{\text{No Action},\text{Additional Testing}\}$,
with binary posterior $\hat p(y\mid x)$ obtained by merging the Pneumonia, COVID-19, and Lung Opacity classes of the Inception-V3 classifier into Non-Normal. We identify the action \emph{No Action} with the Normal label and \emph{Additional Testing} with the Non-Normal label, and use the indicator loss $\ell(a,y)=\mathbf 1\{a\ne y\}$ at threshold $c=1$ (so that the risk constraint becomes $\P(\pi(X)\ne Y)\le \varepsilon$) and
report the mean action cost
$\frac{1}{n_{\text{test}}}\sum_{i=1}^{n_{\text{test}}}
\mathbf 1\{\pi(X_i)=\text{Additional Testing}\}$.
Fixing $\pi_0=\mathrm{RAC}(0.01)$ and $\varepsilon=0.1$, we apply \Cref{alg:crc-nonmono-conservative} and the random-mixing baseline of \eqref{eq:random-mixing} on the same training, calibration, and test splits as above. \Cref{tab:binary-toy-random-mix} reports the two policies. The post-processed policy and random mixing both attain the risk budget, but post-processing reaches it with a smaller mean cost ($0.43$ vs.\ $0.44$) and a smaller switch rate ($0.03$ vs.\ $0.04$).

\begin{table}
    \centering
    \small
    \setlength{\tabcolsep}{4pt}
    \renewcommand{\arraystretch}{1.14}
    \begin{tabular*}{\textwidth}{@{\extracolsep{\fill}}lccc@{}}
        \toprule
        Policy & Violation risk & Mean cost & Switch rate \\
        \midrule
        random-mix & $0.10\se{0.00}$ & \cellcolor{green!15}$0.44\se{0.00}$ & \cellcolor{green!15}$0.04\se{0.01}$ \\
        post-processed & $0.10\se{0.00}$ & \cellcolor{green!30}$0.43\se{0.00}$ & \cellcolor{green!30}$0.03\se{0.00}$ \\
        \bottomrule
    \end{tabular*}
    \vspace{0.1em}
    \caption{\footnotesize Binary diagnostic toy at $\pi_0=\mathrm{RAC}(0.01)$ and $\varepsilon=0.1$. Entries are mean $\pm$ standard error over $20$ random seeds. Reference endpoints: $\pi_0$ has violation risk $0.13\se{0.01}$ and mean cost $0.39\se{0.01}$; $\hat\pi_*$ has violation risk $0.06\se{0.00}$ and mean cost $0.50\se{0.00}$.}
    \label{tab:binary-toy-random-mix}
    \vspace{-2em}
\end{table}

\subsection{LLM thinking-mode routing}\label{subsec:llm-routing}

\paragraph{Setup.}
A natural instance of post-processing arises in modern LLM deployments, where
each query is routed between a cheap, fast model ($\pi_0$) and a more
expensive model run in extended-thinking mode ($\pi_s$). We study this
routing problem on the MMLU-Pro multiple-choice benchmark
\citep{wang2024mmlu}, where $\Y=\A$ is the set of answer choices and
$\ell(a,y)=\mathbf 1\{a\ne y\}$ at threshold $c=1$, so the risk constraint
becomes a bound on the answer-error rate
$\P(\pi(X)\ne Y)\le\varepsilon$. We instantiate
\Cref{alg:crc-nonmono-conservative} with fallback policy $\pi_s=$ Qwen3-32B
in thinking mode and report two fast baselines, $\pi_0=$ Qwen3-4B and
$\pi_0=$ Qwen3-1.7B, both in non-thinking mode
\citep{yang2025qwen3}.\footnote{The thinking model $\pi_s$ attains a test
answer-error rate of $0.234$, so it is not an exact-safe fallback policy, and the
loss curve in the threshold is non-monotone; the expected-risk guarantee of
\Cref{cor:beta-logn-over-n} therefore requires the rank-stability
assumptions of \Cref{thm:EK-log} to hold for the triple
$(\pi_0,\pi_s,\hat\Delta)$; we do not claim to verify these assumptions in this empirical study.} Each model is queried with a fixed sampling
seed so that the policy is a deterministic function of $X$. We measure
compute cost by the per-request forward FLOPs \citep{kaplan2020scaling}
$\kappa(a,x)\;:=\;2\,N_{\text{params}}(a)\,\bigl(L_{\text{prompt}}(x)+L_{\text{completion}}(x,a)\bigr)$,
where $N_{\text{params}}(a)$ is the parameter count of model $a$, and
$L_{\text{prompt}}(x)$ and $L_{\text{completion}}(x,a)$ are the prompt and
generated token counts. We report on the test split the average per-request FLOPs in
teraFLOPs. \footnote{$1\,\text{TFLOP}=10^{12}$ floating-point operations.}

The data are randomly split into training ($5000$), threshold
calibration ($5000$), and test ($2032$) sets, and all reported values are
averages over $20$ independent repetitions. The score is the clipped estimated improvement
$\hat\Delta(x)=(\hat g_0(x)-\hat g_s(x))_+$, where $\hat g_0$ and $\hat g_s$ are obtained from two logistic regressions
trained on the training sample, predicting $\mathbf 1\{\pi_0(X)\ne Y\}$ and
$\mathbf 1\{\pi_s(X)\ne Y\}$ from the concatenation of mid-depth last-token
hidden states of $\pi_0$ and $\pi_s$. We compare against the random-mixing
baseline of \eqref{eq:random-mixing}, applied with $\hat \pi_*\equiv\pi_s$.

\paragraph{Results.}
\Cref{tab:llm-think-routing} reports our algorithm and the random-mixing baseline
for $\pi_0=\text{Qwen3-4B}$ across a grid of risk budgets $\varepsilon$. Both methods empirically attain the target violation risk in every row. At matched violation risk, our algorithm switches less often to the thinking model and is uniformly cheaper, saving $12$--$19\%$ of FLOPs over random mixing. We also report the corresponding $\pi_0=\text{Qwen3-1.7B}$ results in \Cref{tab:llm-think-routing-qwen17}, where the same qualitative pattern is observed.

\begin{table}
    \centering
    \small
    \setlength{\tabcolsep}{3.2pt}
    \renewcommand{\arraystretch}{1.14}
    \begin{tabular*}{\textwidth}{@{\extracolsep{\fill}}lcccccc@{}}
        \toprule
        \multicolumn{7}{c}{\bfseries $\pi_0=\text{Qwen3-4B}$, $\pi_s=\text{Qwen3-32B}$} \\
        \midrule
        & \multicolumn{2}{c}{Violation risk} & \multicolumn{2}{c}{Switch rate} & \multicolumn{2}{c}{FLOPs (T)} \\
        \cmidrule(lr){2-3}\cmidrule(lr){4-5}\cmidrule(lr){6-7}
        $\varepsilon$ & post-processed & random-mix & post-processed & random-mix & post-processed & random-mix \\
        \midrule
        $0.25$ & $0.25\se{0.00}$ & $0.25\se{0.00}$ & \cellcolor{green!30}$0.78\se{0.01}$ & \cellcolor{red!30}$0.91\se{0.00}$ & \cellcolor{green!30}$138.38\se{1.14}$ & \cellcolor{red!30}$157.96\se{0.73}$ \\
        $0.30$ & $0.30\se{0.00}$ & $0.30\se{0.00}$ & \cellcolor{green!30}$0.48\se{0.01}$ & \cellcolor{red!30}$0.63\se{0.00}$ & \cellcolor{green!30}$92.72\se{1.05}$ & \cellcolor{red!30}$114.75\se{0.83}$ \\
        $0.35$ & $0.35\se{0.00}$ & $0.35\se{0.00}$ & \cellcolor{green!30}$0.25\se{0.01}$ & \cellcolor{red!30}$0.36\se{0.00}$ & \cellcolor{green!30}$56.81\se{1.09}$ & \cellcolor{red!30}$69.34\se{0.93}$ \\
        \bottomrule
    \end{tabular*}
    \vspace{0.35em}
    \caption{\footnotesize LLM thinking-mode routing: our algorithm versus the random-mixing baseline across risk budgets $\varepsilon$. Compute is per-request forward FLOPs in TFLOPs. Entries are mean $\pm$ standard error over $20$ random seeds.}
    \label{tab:llm-think-routing}
\end{table}

\vspace{-1em}

\section{Discussion}\label{sec:discussion}

We studied the problem of post-processing a baseline policy to enforce a chance constraint on a downstream loss. 
After characterizing the population optimum,
we showed small excess violation risk under regularity conditions, and proved exact risk control and near-optimality when the fallback is exact-safe. 
The experiments demonstrated that our algorithm achieves risk control with minimally invasive post-processing, comparing favorably to a random mixing baseline.
Future directions include generalizing the framework to randomized policies and continuous action spaces.

\section{Acknowledgements}

This work was supported in part by the US NSF, ARO, AFOSR, ONR, the Simons Foundation
and the Sloan Foundation.

{\small
\setlength{\bibsep}{0.2pt plus 0.3ex}
\IfFileExists{ref.bib}{%
\bibliographystyle{plainnat}
\bibliography{ref}%
}{%
}
}

\appendix

\section{Additional related work}\label{sec:addl-related-work}

\paragraph{Conformal prediction and conformal risk control.}
Conformal prediction provides finite-sample distribution-free predictive guarantees under exchangeability \citep{vovk2005algorithmic,shafer2008tutorial,angelopoulos2023conformal,angelopoulos2024theoretical}. Conformal risk control extends these ideas from miscoverage to more general user-specified risks \citep{bates2021distribution,angelopoulos2021learn,angelopoulos2022conformal}. The standard one-dimensional conformal risk-control argument relies on monotonicity of the loss in a tuning parameter. Our exact-safe fallback result fits exactly into this paradigm, while plug-in calibration leads to non-monotone losses. Recent work has begun to study non-monotone risk control via stability or finite-grid analyses \citep{angelopoulos2026conformalriskcontrolnonmonotonic,aldirawi2026conformal}; our analysis is specialized to policy post-processing and exploits the rank stability of the selected threshold.

\paragraph{Decision-aware conformal methods.}
Several recent works calibrate uncertainty estimates or decision rules directly for downstream decision-making. Conformal predictive decision-making and conformal decision theory calibrate decisions rather than only prediction sets \citep{vovk2018conformal,lekeufack2024conformal}. Decision-focused uncertainty quantification and conformal calibration for optimization under uncertainty incorporate downstream costs and robust-optimization objectives into the construction of uncertainty sets \citep{cortes2024decision,yeh2024end,patel2024conformal,johnstone2021conformal}. Our focus is complementary: rather than learning a new decision rule or uncertainty set from scratch, we seek the least intrusive post-processing of a fixed baseline policy that satisfies a chance constraint.

\paragraph{Reject-option and selective prediction.}
The special case in which the post-processor replaces a risky prediction by an abstention, referral, or additional-test action is closely related to classification with a reject option, originating with Chow's optimal error--reject tradeoff \citep{chow1970optimum}. Later work developed convex surrogates and learning-theoretic analyses for reject-option classifiers \citep{bartlett2008classification,cortes2016learning}, and recent surveys organize this broader literature around ambiguity and novelty rejection \citep{hendrickx2024machine}. Our formulation differs in two ways: it allows an arbitrary finite action space and loss matrix, and it treats the baseline policy as an object to be preserved unless risk control requires otherwise.

\paragraph{Prediction, optimization, and decision-making under uncertainty.}
Our population objective is also connected to classical statistical decision theory \citep{wald1949statistical}, predict-then-optimize pipelines \citep{elmachtoub2022smart}, and optimization under uncertainty \citep{keith2021survey}. The distinctive feature of the present work is the distribution-free calibration layer: the post-processor may use a fitted probabilistic model to rank interventions, but the selected threshold is calibrated using held-out data to control the risk budget.

\section{Additional experimental details}

\subsection{Synthetic multiclass task}
\label{app:synthetic_exp}
\paragraph{Setup.}
In this section, we construct a four-class synthetic problem
in order to illustrate that \Cref{alg:crc-nonmono-conservative} recovers nearly
the same agreement--risk tradeoff 
as the oracle post-processor.

The covariates are sampled according to
$X\sim N(0,I_2)$, the label and action spaces are $\mathcal Y=\mathcal A=\{1,2,3,4\}$, and the labels are drawn from
a well-specified linear-softmax model
\[
p(y=k\mid x)=\frac{\exp(s_k(x))}{\sum_{j=1}^4\exp(s_j(x))},
\]
where, for $x=(x_1,x_2)\in \mathbb R^2$, 
\[
\begin{aligned}
s_1(x)&=1.8x_1-0.4x_2, &
s_2(x)&=-1.2x_1+1.1x_2,\\
s_3(x)&=0.5x_1+1.6x_2, &
s_4(x)&=-0.8x_1-1.3x_2.
\end{aligned}
\]
The baseline policy $\pi_0$ is the
argmax rule from a multinomial logistic regression fit on the training split.
We use the same loss matrix as in Section \ref{subsec:med_diagnosis},
where action $4$ is a conservative fallback. With risk budget
$\varepsilon=0.18$, the cutoff $c$ determines the structure of the fallback
problem. For $c>6$, action $4$ is an exact-safe fallback, in that it obeys the condition in \Cref{eq:pointwise-safety}. For $c\le 6$, no action is exact
safe.

We compare three policies: the baseline $\pi_0$, an oracle post-processor using
the true conditional distribution, 
and \Cref{alg:crc-nonmono-conservative}, which uses fitted class probabilities to
construct $\hat \pi_*$ and $\hat\Delta$. The oracle post-processor computes the oracle fallback policy and
oracle score from the true conditional probabilities, with its threshold
approximated on a large independent reference sample. It is the population
solution of \Cref{thm:zero-one-pop}; its agreement with $\pi_0$ is therefore the
target objective value for the finite-sample algorithm. Each repetition uses
$n_{\mathrm{train}}=250$, $n_{\mathrm{cal}}=200$, and
$n_{\mathrm{test}}=6000$, and results are averaged over $40$ independent repetitions.

\begin{figure}
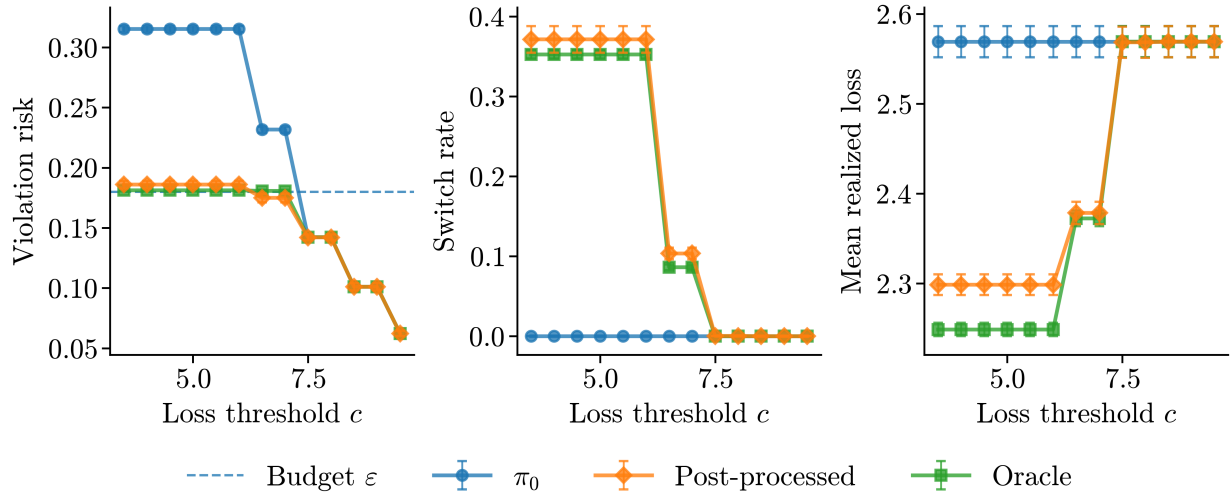

    \centering
    \plotorfallback[width=\textwidth]{plots/summary_panel_well_spec.png}
    \caption{Synthetic multiclass experiment. The panels show violation risk, switch rate, and mean realized loss as a function of the loss cutoff
    $c$ (mean $\pm$ standard error over $40$ repetitions). The dashed line marks the risk budget
    $\varepsilon=0.18$. 
    The oracle curve is the oracle
    policy from \Cref{thm:zero-one-pop}; the algorithm's curve uses a
    fitted fallback and fitted score.}
    \label{fig:summary_panel_well_spec}
\end{figure}

\paragraph{Results.}
\Cref{fig:summary_panel_well_spec} shows that the algorithm closely tracks the
oracle objective, namely agreement with the baseline subject to the risk budget.
When the baseline already satisfies the risk constraint, both the oracle and the algorithm keep high agreement
with $\pi_0$. When the baseline first becomes infeasible in the 
region $c>6$, the algorithm uses action $4$ as the fallback action $\hat\pi_*$ and
intervenes only on high-score contexts. For example, at $c=7$, the baseline risk
is about $0.232$, while the algorithm reduces risk to about $0.175$ with
agreement about $0.896$, close to the oracle risk $0.181$ and oracle agreement
$0.914$.

For $c\le6$, action $4$ no longer obeys the exact-safety condition \Cref{eq:pointwise-safety},
but the algorithm remains close to
this oracle benchmark. Over $3.5\le c\le6.0$, the oracle has risk near the
budget with agreement about $0.647$ and mean loss about $2.249$, while the
algorithm attains risk about $0.186$, agreement about $0.629$, and mean loss
about $2.299$. Thus the algorithm nearly matches the oracle agreement objective
while nearly meeting the risk budget, even in the non-monotone fallback regime.

\paragraph{Computational resources.}
This experiment was run in Google Colab using the default CPU runtime without specialized hardware.

\subsection{Additional COVID-19 diagnosis results}
\label{app:additional-covid-results}
\paragraph{Evaluation metrics.}
We evaluate the following metrics on the test set
$\left\{\left(X_i, Y_i\right)\right\}_{i=1}^{n_{\text {test }}}$:
\begin{itemize}
    \item $\text{Violation risk}:=\frac{1}{n_{\text{test}}}\sum_{i=1}^{n_{\text{test}}}\mathbf{1}\!\bigl\{\ell(\pi(X_i),Y_i)\ge c\bigr\}$
    \item $\text{Mean realized loss}:=\frac{1}{n_{\text{test}}}\sum_{i=1}^{n_{\text{test}}}\ell(\pi(X_i),Y_i)$
    \item $\text{Switch rate}:=\frac{1}{n_{\text{test}}}\sum_{i=1}^{n_{\text{test}}}\mathbf{1}\!\bigl\{\pi(X_i)\ne \pi_0(X_i)\bigr\}$.
\end{itemize}

The two tables below supplement the COVID-19 diagnosis results in
\Cref{subsec:med_diagnosis}. \Cref{tab:covid-nonmonotonic-rac002} repeats the
non-monotone plug-in experiment with the less conservative baseline
$\pi_0=\mathrm{RAC}(0.02)$. The qualitative behavior matches the main
$\mathrm{RAC}(0.01)$ results: the post-processed policy stays near the target
risk $\varepsilon=0.1$ and uses only modest switching, with mean realized loss
below or close to the competing endpoints. \Cref{tab:covid-random-mix-rac002}
compares the same post-processed policy with score-blind random mixing; targeted
switching again attains comparable risk with lower switch rate and lower mean
realized loss.

\begin{table}
    \centering
    \small
    \setlength{\tabcolsep}{2.7pt}
    \renewcommand{\arraystretch}{1.14}
    \begin{tabular*}{\textwidth}{@{\extracolsep{\fill}}lccccccc@{}}
        \toprule
        \multicolumn{8}{c}{\bfseries Supplementary COVID results: $\pi_0=\mathrm{RAC}(0.02)$, $\varepsilon=0.1$} \\
        \midrule
        & \multicolumn{3}{c}{Violation risk} & \multicolumn{3}{c}{Mean realized loss} & \multicolumn{1}{c}{Switch rate} \\
        \cmidrule(lr){2-4}\cmidrule(lr){5-7}
        $c$ & post-processed & $\pi_0$ & $\hat\pi_*$ & post-processed & $\pi_0$ & $\hat\pi_*$ & \\
        \midrule
        $4$ & $0.10\se{0.00}$ & $0.11\se{0.01}$ & $0.06\se{0.00}$ & \cellcolor{green!30}$0.73\se{0.02}$ & \cellcolor{green!15}$0.79\se{0.03}$ & \cellcolor{red!30}$1.01\se{0.01}$ & $0.01\se{0.01}$ \\
        $3$ & $0.10\se{0.00}$ & $0.12\se{0.01}$ & $0.06\se{0.00}$ & \cellcolor{green!30}$0.70\se{0.02}$ & \cellcolor{green!15}$0.79\se{0.03}$ & \cellcolor{red!30}$0.84\se{0.01}$ & $0.02\se{0.01}$ \\
        $2$ & $0.10\se{0.00}$ & $0.15\se{0.01}$ & $0.07\se{0.00}$ & \cellcolor{green!15}$0.61\se{0.01}$ & \cellcolor{red!30}$0.79\se{0.03}$ & \cellcolor{green!30}$0.54\se{0.01}$ & $0.06\se{0.01}$ \\
        \bottomrule
    \end{tabular*}
    \vspace{0.35em}
    \caption{\footnotesize Supplementary COVID-19 diagnosis results for the non-monotone plug-in regime with $\pi_0=\mathrm{RAC}(0.02)$. Entries are mean $\pm$ standard error over $20$ random seeds. The post-processed columns report the policy returned by the algorithm; switch rates are measured relative to $\pi_0$.}
    \label{tab:covid-nonmonotonic-rac002}
\end{table}

\begin{table}
    \centering
    \small
    \setlength{\tabcolsep}{2.7pt}
    \renewcommand{\arraystretch}{1.14}
    \begin{tabular*}{\textwidth}{@{\extracolsep{\fill}}lcccccc@{}}
        \toprule
        \multicolumn{7}{c}{\bfseries Supplementary random-mixing comparison: $\pi_0=\mathrm{RAC}(0.02)$, $\varepsilon=0.1$} \\
        \midrule
        & \multicolumn{2}{c}{Violation risk} & \multicolumn{2}{c}{Switch rate} & \multicolumn{2}{c}{Mean realized loss} \\
        \cmidrule(lr){2-3}\cmidrule(lr){4-5}\cmidrule(lr){6-7}
        $c$ & random-mix & post-processed & random-mix & post-processed & random-mix & post-processed \\
        \midrule
        $4$ & $0.10\se{0.00}$ & $0.10\se{0.00}$ & \cellcolor{green!15}$0.05\se{0.02}$ & \cellcolor{green!30}$0.01\se{0.01}$ & \cellcolor{green!15}$0.81\se{0.03}$ & \cellcolor{green!30}$0.73\se{0.02}$ \\
        $3$ & $0.10\se{0.00}$ & $0.10\se{0.00}$ & \cellcolor{green!15}$0.07\se{0.01}$ & \cellcolor{green!30}$0.02\se{0.01}$ & \cellcolor{green!15}$0.78\se{0.02}$ & \cellcolor{green!30}$0.70\se{0.02}$ \\
        $2$ & $0.10\se{0.00}$ & $0.10\se{0.00}$ & \cellcolor{green!15}$0.07\se{0.01}$ & \cellcolor{green!30}$0.06\se{0.01}$ & \cellcolor{green!15}$0.64\se{0.01}$ & \cellcolor{green!30}$0.61\se{0.01}$ \\
        \bottomrule
    \end{tabular*}
    \vspace{0.35em}
    \caption{\footnotesize Supplementary COVID-19 random-mixing comparison for $\pi_0=\mathrm{RAC}(0.02)$. Entries are mean $\pm$ standard error over $20$ random seeds. The mixing weight $\hat p_{\mathrm{mix}}$ is chosen on the calibration split as the largest value satisfying $\hat p_{\mathrm{mix}}\hat r_0+(1-\hat p_{\mathrm{mix}})\hat r_*\le\varepsilon$.}
    \label{tab:covid-random-mix-rac002}
\end{table}

\subsection{Additional LLM thinking-mode routing results}
\label{app:additional-llm-routing-results}
\Cref{tab:llm-think-routing-qwen17} reports the same thinking-mode routing
experiment as in \Cref{subsec:llm-routing}, but with the smaller fast baseline
$\pi_0=\text{Qwen3-1.7B}$. 
The same pattern persists: at nearly matched empirical violation risk, the score-based
post-processor switches to the thinking model less often than random mixing and
therefore uses fewer FLOPs.

\begin{table}
    \centering
    \small
    \setlength{\tabcolsep}{2.7pt}
    \renewcommand{\arraystretch}{1.14}
    \begin{tabular*}{\textwidth}{@{\extracolsep{\fill}}lcccccc@{}}
        \toprule
        \multicolumn{7}{c}{\bfseries $\pi_0=\text{Qwen3-1.7B}$, $\pi_s=\text{Qwen3-32B}$} \\
        \midrule
        & \multicolumn{2}{c}{Violation risk} & \multicolumn{2}{c}{Switch rate} & \multicolumn{2}{c}{FLOPs (T)} \\
        \cmidrule(lr){2-3}\cmidrule(lr){4-5}\cmidrule(lr){6-7}
        $\varepsilon$ & post-processed & random-mix & post-processed & random-mix & post-processed & random-mix \\
        \midrule
        $0.25$ & $0.25\se{0.00}$ & $0.25\se{0.00}$ & \cellcolor{green!30}$0.89\se{0.00}$ & \cellcolor{red!30}$0.95\se{0.00}$ & \cellcolor{green!30}$153.91\se{0.66}$ & \cellcolor{red!30}$163.80\se{0.57}$ \\
        $0.28$ & $0.28\se{0.00}$ & $0.28\se{0.00}$ & \cellcolor{green!30}$0.76\se{0.00}$ & \cellcolor{red!30}$0.86\se{0.00}$ & \cellcolor{green!30}$133.31\se{0.76}$ & \cellcolor{red!30}$148.41\se{0.59}$ \\
        $0.30$ & $0.30\se{0.00}$ & $0.30\se{0.00}$ & \cellcolor{green!30}$0.69\se{0.00}$ & \cellcolor{red!30}$0.80\se{0.00}$ & \cellcolor{green!30}$121.23\se{0.75}$ & \cellcolor{red!30}$138.64\se{0.61}$ \\
        $0.33$ & $0.33\se{0.00}$ & $0.33\se{0.00}$ & \cellcolor{green!30}$0.59\se{0.00}$ & \cellcolor{red!30}$0.70\se{0.00}$ & \cellcolor{green!30}$104.85\se{0.77}$ & \cellcolor{red!30}$122.16\se{0.60}$ \\
        $0.36$ & $0.36\se{0.00}$ & $0.36\se{0.00}$ & \cellcolor{green!30}$0.50\se{0.00}$ & \cellcolor{red!30}$0.61\se{0.00}$ & \cellcolor{green!30}$89.75\se{0.71}$ & \cellcolor{red!30}$106.85\se{0.84}$ \\
        $0.39$ & $0.39\se{0.00}$ & $0.40\se{0.00}$ & \cellcolor{green!30}$0.41\se{0.00}$ & \cellcolor{red!30}$0.52\se{0.00}$ & \cellcolor{green!30}$75.24\se{0.62}$ & \cellcolor{red!30}$90.93\se{0.68}$ \\
        $0.45$ & $0.45\se{0.00}$ & $0.45\se{0.00}$ & \cellcolor{green!30}$0.25\se{0.00}$ & \cellcolor{red!30}$0.33\se{0.00}$ & \cellcolor{green!30}$48.54\se{0.74}$ & \cellcolor{red!30}$61.53\se{0.79}$ \\
        $0.50$ & $0.50\se{0.00}$ & $0.50\se{0.00}$ & \cellcolor{green!30}$0.13\se{0.00}$ & \cellcolor{red!30}$0.18\se{0.00}$ & \cellcolor{green!30}$27.53\se{0.64}$ & \cellcolor{red!30}$35.60\se{0.79}$ \\
        \bottomrule
    \end{tabular*}
    \vspace{0.35em}
    \caption{\footnotesize Supplementary LLM thinking-mode routing results for our algorithm versus the random-mixing baseline with $\pi_0=\text{Qwen3-1.7B}$. Compute is per-request forward FLOPs in TFLOPs. Entries are mean $\pm$ standard error over $20$ random seeds.}
    \label{tab:llm-think-routing-qwen17}
\end{table}

\subsection{Existing Assets, Licenses, and Terms of Use}
\label{app:assets-licenses}

We used the following existing datasets, model architectures, and pretrained model
weights. We cite the original scientific sources in the main text and record the
corresponding licenses or terms of use here. We do not redistribute any third-party
datasets or third-party pretrained model weights.

\paragraph{COVID-19 Radiography Database.}
The COVID-19 radiograph diagnosis experiments use the COVID-19 Radiography Database
\citep{chowdhury2020can,rahman2021exploring}. The Kaggle listing for this dataset
identifies the license/ownership field as \emph{Data files \copyright{} Original
Authors}. This is not a standard open-source or Creative Commons license. We used
the data only for research evaluation and do not redistribute the images; users of
any released code must obtain the dataset from the original source and comply with
the Kaggle listing and original-author terms.

\paragraph{Inception-V3, TorchVision, and ImageNet-pretrained weights.}
For the COVID-19 image classifier, we initialized Inception-V3
\citep{szegedy2015going,szegedy2016rethinking} using the TorchVision implementation
\citep{torchvision2016}, specifically the
\texttt{torchvision.models.inception\_v3} model with
\texttt{Inception\_V3\_Weights.IMAGENET1K\_V1}. TorchVision is distributed under
the BSD 3-Clause License. TorchVision notes that pretrained weights may have
additional licenses or terms derived from the dataset used for training. The
pretrained weights used here are ImageNet-1K pretrained weights
\citep{deng2009imagenet}; the ImageNet access terms restrict use of the ImageNet
database to non-commercial research and educational purposes. We use these weights
only for model initialization and do not redistribute ImageNet data or TorchVision
pretrained weights.

\paragraph{MMLU-Pro.}
The LLM routing experiments use the MMLU-Pro benchmark \citep{wang2024mmlu},
accessed through the Hugging Face dataset identifier
\texttt{TIGER-Lab/MMLU-Pro}. The Hugging Face dataset card lists the MMLU-Pro
dataset license as MIT. We use the benchmark for evaluation and do not redistribute
a modified copy of the dataset.

\paragraph{Qwen3 models.}
The LLM routing experiments use the Qwen3 model family \citep{yang2025qwen3},
accessed through the Hugging Face model identifiers
\texttt{Qwen/Qwen3-1.7B}, \texttt{Qwen/Qwen3-4B}, and
\texttt{Qwen/Qwen3-32B}. The Hugging Face model cards for these three repositories
list the license as Apache License 2.0. We use these models for inference only and
do not redistribute the model weights.

\section{Additional theoretical components}

\subsection{Constants appearing in the main text}
\label{app:constants}

For \Cref{thm:EK-log,cor:beta-logn-over-n}, define
\[
p_s:=\P(\ell(\hat\pi_*(X),Y)\ge c),
\qquad
p_0:=\P(\ell(\pi_0(X),Y)\ge c),
\qquad
\kappa:=\varepsilon-p_s,
\qquad
\xi:=p_0-\varepsilon .
\]
Let
\[
c_0:=\frac{\kappa}{4},
\qquad
\zeta>0
\quad\text{satisfy}\quad
\P(\hat\Delta(X)<\zeta)\le \frac{\kappa}{8},
\qquad
\lambda_W:=c_W\zeta^\beta .
\]
Then, for \Cref{thm:EK-log,cor:beta-logn-over-n}, one may take
\[
C_1:=\frac{16}{\lambda_W^2},
\qquad
C_2:=\frac{2}{\lambda_W}+3,
\qquad 
C_3:=3C_1.
\]
For \Cref{thm:near-opt-short}, 
with $C_R, L_J$ defined as in \Cref{thm:near-opt}, one may take
\[
C_4:=\frac{\sqrt{2}\,L_J}{C_R},
\qquad
C_5:=\frac{2L_J}{C_R},
\qquad
C_6:=1+\frac{L_J}{C_R}.
\]

\subsection{Preliminaries for non-monotone risk control guarantee}
\label{app:formal-rank-stability}

This appendix records definitions and preliminary results used in the proofs of results in \Cref{subsubsec:crc-nonmono-stability}. 
For a dataset
$D=\{z_r=(x_r,y_r):r\in [m]\}$, define
\[
\mathcal T(D):=\{0,\top\}\cup\{\hat\Delta(x_r):r\in [m]\},
\qquad
\hat R_D^+(\tau):=\frac{1}{m+1}\left[\sum_{z\in D}L(z;\tau)+1\right],
\]
and set
\[
\Select(D):=\max\{\tau\in\mathcal T(D):\hat R_D^+(\tau)\le\varepsilon\},
\]
with the convention $\Select(D)=0$ when the displayed set is empty. 
Given a subset $S\subseteq [m]$, we write $D_S := \{ z_r : r\in S \}$ for the sub-dataset of observations whose indices lie in $S$.

Let $N:=n+1$ and consider the augmented sample $\{ Z_i=(X_i,Y_i) : i\in[N] \}$, where $Z_{[n]}$ are calibration observations and $Z_{n+1}$ is the test point. Fix a fitted plug-in score $\hat \Delta$ obtained via \Cref{eq:plug-in-ests}, and assume $\hat \Delta$ is independent of the augmented sample.
Write $\hat\Delta_i:=\hat\Delta(X_i)$ for $i\in [N]$.
Assume $\hat\Delta(X)$ has an atomless distribution, so that $\hat\Delta_1$,
$\dots$,
$\hat\Delta_N$ have no ties a.s.
Let $\hat\Delta_{(1)}<\cdots<\hat\Delta_{(N)}$ be the order statistics.
For $i\in [N]$, let
$q_i\in[N]$ be the unique rank such that $\hat\Delta_i=\hat\Delta_{(q_i)}$.
Equivalently, for each $j\in [N]$, we may define the unique rank $v_j\in [N]$ such that
$\hat\Delta_{v_j}=\hat\Delta_{(j)}$.
By convention, we define 
$\hat \Delta_{(0)} := 0 $, 
$\hat\Delta_{(N+1)}:=\top$,
$q_0 := 0$,
and $q_{N+1} := N+1$.
For convenience, also introduce the notation 
$\hat R_{1:N}^+(\cdot) := \hat R_{Z_{[N]}}^+(\cdot)$ for the bumped empirical risk function over the augmented sample, 
and $\hat R_{-i}^+(\cdot) := \hat R_{Z_{[N]\setminus\{i\}}}^+$ for the bumped empirical risk functions over the leave-one-out sample with observation $i$ removed, 
for each $i\in [N]$.

Define the \textit{augmented threshold} as $\hat\tau_{1:N}:=\Select(Z_{[N]})$.
For $i\in [N]$, define the \textit{leave-one-out threshold} corresponding to observation $i$ as $\hat\tau_{-i}:=\Select(Z_{[N]\setminus\{i\}})$.

Next, we define the \textit{augmented threshold rank} $\hat j$ as follows.
If $\hat \tau_{1:N} = \hat \Delta_k$ for some $k\in [N]$, then $\hat j := q_k$.
If $\hat \tau_{1:N} = 0$, then $\hat j := 0$. 
If $\hat \tau_{1:N} = \top$, then $\hat j := N+1$.
(Since $\hat \Delta(X)$ has an atomless distribution on $[0,1]$, the case $\hat \Delta_{(1)} = 0$ occurs with probability zero, hence $\hat j$ is well-defined a.s.)
It follows that $\mathbf 1\{\hat\Delta_i<\hat\tau_{1:N}\}=\mathbf 1\{q_i<\hat j\}$ a.s.\ for all $i\in \{0,\ldots,N+1\}$.

Similarly, given $i\in [N]$, we define the \textit{leave-one-out threshold rank} $\hat j_{-i}$ associated to observation $i$ as follows.
If $\hat \tau_{-i} = \hat \Delta_k$ for some $k\in [N]\setminus \{i\}$, then $\hat j_{-i} := q_k$.
If $\hat \tau_{-i} = 0$, then $\hat j_{-i} := 0$. 
If $\hat \tau_{-i} = \top$, then $\hat j_{-i} := N+1$.
It follows that $\mathbf 1\{\hat\Delta_k<\hat\tau_{-i}\}=\mathbf 1\{q_k<\hat j_{-i}\}$ a.s.\ for all $i,k\in \{0,\ldots,N+1\}$.
Note that if we define \[
\mathcal J_{-i}:=\{q_k : k\in \{0,\ldots,N+1\}\setminus\{i\}\},
\]
then $\hat j_{-i}\in \mathcal J_{-i}$ for all $i\in [N]$.
Finally, for $i\in [N]$, if $\hat j_{-i} < N+1$, we define
\begin{align}\label{eq:j-minus-i-plus}
    \hat j_{-i}^+:=\min\{j\in\mathcal J_{-i}:j>\hat j_{-i}\}
\end{align}
as the smallest index greater than $\hat j_{-i}$ among the set of leave-one-out indices $\mathcal J_{-i}$.

With these definitions in place, we may define the \textit{rank stability} parameter $K$ as \begin{align}\label{eq:rank-stab}
    K := \max_{i\in[N]}|\hat j_{-i}-\hat j|.
\end{align}
The following result relates $K$ to the difference in losses between leave-one-out and augmented thresholds, and can be seen as a variant of \cite[Proposition 4]{angelopoulos2026conformalriskcontrolnonmonotonic}, which presents a similar bound for a selective classification algorithm.

\begin{proposition}\label{prop:stab-K}
Under the conditions in \Cref{cor:beta-logn-over-n}, we have
\[
\frac{1}{N}\sum_{i=1}^N \big|L(Z_i;\hat\tau_{-i})-L(Z_i;\hat\tau_{1:N})\big|
\le \frac{2K}{N}
\qquad\text{a.s.}
\]
\end{proposition}

The proof of the above result is provided in Section \ref{sec:proof-prop-stab-k}.

\paragraph{A bound on $\E[K]$.}
For $i\in [N]$, define the random variables
\[
I^0_i:=\mathbf 1\{\ell(\pi_0(X_i),Y_i)\ge c\},
\qquad
I^s_i:=\mathbf 1\{\ell(\hat \pi_*(X_i),Y_i)\ge c\},
\qquad
W_i:=I^0_i-I^s_i,
\]
so that the violation loss may be written as
\begin{equation}\label{eq:L-nonmono-stab}
L(Z_i;\tau)=I^s_i+W_i\mathbf 1\{\hat\Delta(X_i)<\tau\}.
\end{equation}
Define the concomitants
$I^s_{(j)}:=I^s_{v_j}$ and $W_{(j)}:=W_{v_j}$ for $j\in[N]$.\footnote{The parentheses distinguish rank indices from the original observation indices.}
For $j\in\{0,1,\dots,N+1\}$, write $\tau_j := \hat \Delta_{(j)}$.

The loss sum over the augmented sample evaluated at $\tau = \tau_j$ is
\[
\sum_{k=1}^N L(Z_k;\tau_j)
=
\sum_{t=1}^N I^s_{(t)}+\sum_{t=1}^N W_{(t)}\mathbf 1\{t<j\},
\]
where the second sum is zero for $j=0$ and $j=1$, and is $\sum_{t=1}^N W_{(t)}$ for $j=N+1$.
We also define the centered partial sums
\[
T_j:=\sum_{t=1}^N (I^s_{(t)}-\varepsilon)+\sum_{t=1}^N W_{(t)}\mathbf 1\{t<j\}+(1-\varepsilon).
\]
Similarly, given $i\in [N]$ and $j\in \{0,\ldots,N+1\}$, the loss sum over the leave-one-out sample that removes observation $i$ evaluated at $\tau = \tau_j$ is 
\[
\sum_{k\in [N]\setminus \{i\}} L(Z_k;\tau_j)
=
\sum_{t\in [N]\setminus \{q_i\}} I^s_{(t)}
+\sum_{t\in [N]\setminus \{q_i\}} W_{(t)}\mathbf 1\{t<j\},
\]
and we define the centered partial sums
\[
T_j^{-i}:=
\sum_{t\in [N]\setminus \{q_i\}} (I^s_{(t)} - \ep)
+\sum_{t\in [N]\setminus \{q_i\}} W_{(t)}\mathbf 1\{t<j\}
+ (1-\ep).
\]
Note that $T_j^{-i}$ is related to $T_j$ via the identity
\[
T_j^{-i}:=T_j-(I^s_i-\varepsilon)-W_i\mathbf 1\{q_i<j\},
\qquad j\in\{0,1,\dots,N+1\}.
\]
We say that a threshold $\tau\in [0,1]\cup \{\top\}$ is \textit{$\hat R_{1:N}^+$-feasible} if $\hat R_{1:N}^+(\tau)\le \ep$.
Given $i\in [N]$, we say that a threshold $\tau\in [0,1]\cup \{\top\}$ is \textit{$\hat R_{-i}^+$-feasible} if $\hat R_{-i}^+(\tau)\le \ep$.
Likewise, we say that an index $j\in \{0,\ldots, N+1\}$ is \textit{$\hat R_{1:N}^+$-feasible} if $\hat R_{1:N}^+(\tau_j)\le \ep$.
Given $i\in [N]$, we say that an index $j\in \mathcal J_{-i}$ is \textit{$\hat R_{-i}^+$-feasible} if $\hat R_{-i}^+(\tau_j)\le \ep$.
It is easy to see that an index $j\in \{0,\ldots,N+1\}$ is $\hat R_{1:N}^+$-feasible iff $T_j\le0$, because
$
\hat R_{1:N}^+(\tau_j)-\varepsilon
=
\frac{T_j}{N+1}.
$
Similarly, an index $j\in\mathcal J_{-i}$ is $\hat R_{-i}^+$-feasible iff $T_j^{-i}\le0$, because
$
\hat R_{-i}^+(\tau_j)-\varepsilon
=
\frac{T_j^{-i}}{N}.
$

\begin{lemma}[Uniform perturbation bound]\label{lem:T-perturb}
Under the conditions in \Cref{cor:beta-logn-over-n}, for every $i\in[N]$ and every $j\in\{0,1,\dots,N+1\}$,
$|T_j^{-i}-T_j|\le 1$ a.s.
\end{lemma}

The proof of the above result is provided in Section \ref{sec:proof-lem-t-perturb}.

\begin{lemma}[One-step bounds]\label{lem:endpoint-crossing}
Assume the conditions in \Cref{thm:EK-log} hold, and assume $N\ge 2$. For each \(i\in[N]\), define
\[
a_i:=\min\bigl\{\mathcal J_{-i}\setminus\{0,N+1\}\bigr\},
\qquad
b_i:=\max\bigl\{\mathcal J_{-i}\setminus\{N+1\}\bigr\}.
\]
Then
$T_1=T_0$,
$T^{-i}_{a_i}=T^{-i}_0$,
$|T_{N+1}-T_N|\le 1$,
and
$|T^{-i}_{N+1}-T^{-i}_{b_i}|\le1$.
If we also have \(0<\hat j<N+1\), then
$|T_{\hat j+1}-T_{\hat j}|\le 1$.
Similarly, if for some $i\in [N]$ we have
\(0<\hat j_{-i}<N+1\), then 
$|T^{-i}_{\hat j^+_{-i}}-T^{-i}_{\hat j_{-i}}|\le 1$,
where $\hat j_{-i}^+$ is defined in \Cref{eq:j-minus-i-plus}.
\end{lemma}

The proof of the above result is provided in Section \ref{sec:proof-lem-one-step}.

\begin{lemma}[Low-sum block between $\hat j$ and $\hat j_{-i}$]\label{lem:block-3}
Assume the augmented threshold rank obeys $\hat j<N+1$ and $T_{\hat j}\le0<T_{\hat j+1}$.
For each $i\in [N]$, assume the leave-one-out threshold rank associated to observation $i$ obeys $\hat j_{-i}<N+1$ and $T^{-i}_{\hat j_{-i}}\le0<T^{-i}_{\hat j_{-i}^+}$, where $\hat j_{-i}^+$ is defined in \Cref{eq:j-minus-i-plus}.
Then, under the conditions in \Cref{cor:beta-logn-over-n}, we have
\[
\sum_{t=\max\{1,\min\{\hat j,\hat j_{-i}\}\}}^{\max\{\hat j,\hat j_{-i}\}-1} W_{(t)} \le 1,
\]
where we use the convention that an empty sum equals zero.
\end{lemma}

The proof of the above result is provided in Section \ref{sec:proof-lem-block-3}.

\subsection{Full near-optimality theorem}
\label{app:near-optimality-full}

\begin{theorem}[Near-optimality of exact-safe fallback post-processing]
\label{thm:near-opt}
Fix \(\delta\in(0,1)\). Suppose the following conditions hold.
\begin{enumerate}
    \item The fitted fallback policy is exact-safe: 
    \[
    \mathbf 1\{\ell(\hat \pi_*(x),y)\ge c\}=0
    \quad\text{for all } x\in \mathcal X, y\in \mathcal Y.
    \]
    \item The calibration and test observations \((X_i,Y_i)_{i=1}^{n+1}\) are i.i.d. 
    \item The oracle risk curve
    \[
    R^*(\tau):=\P(\ell(\pi_0(X),Y)\ge c,\Delta(X)<\tau)
    \]
    for $\tau\in [0,1]$ has an interior crossing: there exists \(\tau^*\in(0,1)\) satisfying
    \(R^*(\tau^*)=\varepsilon\). Moreover, there are constants \(C_R>0\) and
    \(\rho\in(0,\min\{\tau^*,1-\tau^*\})\) such that
    \[
    |R^*(\tau)-R^*(\tau^*)|
    \ge C_R|\tau-\tau^*|
    \]
    whenever \(|\tau-\tau^*|\le\rho\).
    \item The oracle agreement curve is locally Lipschitz at \(\tau^*\): there
    is a constant \(L_J<\infty\) such that
    \[
    \big|\P(\Delta(X)<\tau)-\P(\Delta(X)<\tau^*)\big|
    \le L_J|\tau-\tau^*|
    \]
    whenever \(|\tau-\tau^*|\le\rho\).
\end{enumerate}
Define
\[
\mathcal E_\Delta(u):=
\P(|\hat\Delta(X)-\Delta(X)|>u)
+\sup_{t\in[0,1]}\P(|\Delta(X)-t|\le u),
\qquad
\ep_\Delta:=\inf_{u>0}\mathcal E_\Delta(u).
\]
Also define
\[
\hat p:=\frac1n\sum_{i=1}^n\mathbf 1\{\ell(\pi_0(X_i),Y_i)\ge c\},
\qquad
\varepsilon_{1,n}(\delta):=
\sqrt{\frac{\log(4/\delta)}{2n}}
+\sqrt{\frac{\hat p\log(4/\delta)}{2n}},
\]
and
\[
\varepsilon_n(\delta):=
\frac{n}{n+1}\varepsilon_{1,n}(\delta)
+\frac{1}{n+1}
+\ep_\Delta,
\qquad
\nu_n(\delta):=
C_R^{-1}\left(\varepsilon_n(\delta)+\frac{1}{n+1}\right),
\]
\[
\eta_n(\delta):=L_J\nu_n(\delta)+\ep_\Delta.
\]
With probability at least \(1-\delta\) over the calibration data, if \(\nu_n(\delta)\le\rho\), then
\[
\left(J^*-\P\bigl(\hat\pi(X;\hat\tau)=\pi_0(X)\bigr)\right)_+
\le \eta_n(\delta),
\]
where \(J^*:=\P(\Delta(X)<\tau^*)\) is the optimal population agreement value for the problem in
\Cref{eq:zero-one-obj}. 
Consequently, we have the bound
\[ 
\left(J^*-\P\bigl(\hat\pi(X;\hat\tau)=\pi_0(X)\bigr)\right)_+
\le
C_4\sqrt{\frac{\log(4/\delta)}{n}}
+
\frac{C_5}{n+1}
+
C_6\varepsilon_\Delta, \]
where $C_4, C_5, C_6$ are defined in \Cref{app:constants}.
In addition, whenever \(\ep_\Delta\to0\), the sub-optimality positive-part is $O(n^{-1/2})+O(\ep_\Delta)$.
\end{theorem}

The proof of the above result is provided in \Cref{sec:proof-thm-near-opt}.

\begin{remark}[Threshold-disagreement modulus]
The quantity \(\mathcal E_\Delta(u)\) measures how often thresholding
\(\hat\Delta(X)\) can disagree with thresholding the oracle score \(\Delta(X)\),
uniformly over all thresholds. Indeed, for any \(\tau\in[0,1]\),
\[
\P\!\left(
\mathbf 1\{\hat\Delta(X)<\tau\}\neq \mathbf 1\{\Delta(X)<\tau\}
\right)
\le
\P(|\hat\Delta(X)-\Delta(X)|>u)
+\P(|\Delta(X)-\tau|\le u)
\le \mathcal E_\Delta(u).
\]
The first term is the score estimation error, while the second is the mass of
oracle scores within distance \(u\) of a threshold, i.e., points whose decision
is sensitive to perturbations of size \(u\). Thus
\(\ep_\Delta=\inf_{u>0}\mathcal E_\Delta(u)\) is a uniform thresholding error
bound. If \(|\hat\Delta(X)-\Delta(X)|\le a_n\) a.s.\ and
\(\Delta(X)\) has bounded density, then \(\ep_\Delta=O(a_n)\).
\end{remark}

\subsection{Adaptive concentration for exact-safe fallback}

\begin{corollary}[Adaptive concentration in the exact-safe fallback setting]\label{cor:step1-eps1}
Assume  the exact-safe fallback condition holds: $\mathbf 1\{\ell(\hat \pi_*(x),y)\ge c\}=0$ for all $x\in \mathcal X, y\in \mathcal Y$. Assume that the calibration samples \((X_i,Y_i)_{i=1}^n\) are i.i.d.
For $i\in [n]$, define
$W_i:=\mathbf 1\{\ell(\pi_0(X_i),Y_i)\ge c\}$ and $Z_i:=\hat\Delta(X_i)$,
so that
$L(Z_i;\tau)=W_i\,\mathbf 1\{Z_i<\tau\}$
a.s.\ for all $\tau\in[0,1]$.
Let $p:=\mathbb{P}(W_i=1)$, $m:=\sum_{i=1}^n W_i$, and $\hat p:=m/n$.
Let
$\hat R(\tau):=\frac{1}{n}\sum_{i=1}^n L(X_i,Y_i;\tau)$ and $R(\tau):=\E[L(X_i,Y_i;\tau)]$.
Then for any $\delta\in(0,1)$, with probability at least $1-\delta$,
\[
\sup_{\tau\in[0,1]}|\hat R(\tau)-R(\tau)|
\le
\varepsilon_{1,n}(\delta) 
:= \sqrt{\frac{\log(4/\delta)}{2n}}
+
\sqrt{\frac{\hat p\log(4/\delta)}{2n}}.
\]
\end{corollary}

The proof of the above result is provided in Section \ref{sec:proof-cor-step1-eps1}.

\subsection{Low-sum block bounds for drifted sequences}
\label{app:low-sum-blocks}

For independent bounded increments with uniformly positive drift on a suffix, we bound the
expected length of the longest contiguous block whose sum is small. This is used in
\Cref{thm:EK-log}.

\begin{corollary}[Suffix low-sum block bound]\label{cor:low-sum-b}
Let $(S_t)_{t=1}^n$ be independent random variables with $S_t\in[-1,1]$ a.s.
Fix an index $i_0\in\{1,\dots,n\}$ and assume there exists $\lambda>0$ such that
$\mathbb E[S_t]\ge \lambda$ for all $t\ge i_0$.
For any $b\ge0$, define
\[
M^{(b)}_{i_0}(n)
:=
\max\Bigl(\{0\}\cup\Bigl\{h\in [n] : \ \exists\, r\in\{i_0,\dots,n-h+1\}\ \text{s.t.}
\sum_{t=r}^{r+h-1} S_t \le b
\Bigr\}\Bigr).
\]
Then, for all sufficiently large $n$,
$\mathbb E\!\left[M^{(b)}_{i_0}(n)\right] \le \frac{16}{\lambda^2}\log n+\frac{2b}{\lambda}+2$.
\end{corollary}

The proof of the above result is provided in Section \ref{sec:proof-cor-low-sum-b}.

\subsection{Interior crossing with high probability}
\label{subsubsec:interior-crossing}

\begin{lemma}[Right-endpoint exclusion for augmented and leave-one-out threshold ranks]\label{lem:Eint-all-high-prob}
Assume the conditions in \Cref{thm:EK-log} hold.
Let $N\ge2$, and assume that $I^0_1,\dots,I^0_N$ are i.i.d.~with
$p_0:=\mathbb{P}(I^0_1=1)\ge \varepsilon+\xi$
for some $\xi>0$.
For $i\in [N]$, define the quantity
\[ \hat p_{0,-i}:=\frac{1}{N-1}\sum_{k\in [N]\setminus \{i\}} I^0_k \]
and the events
\[
E_{\mathrm{int},-i}:=\Big\{\hat p_{0,-i}>\varepsilon+\frac{\xi}{2}\Big\},
\qquad
E_{\mathrm{int}}:=\bigcap_{i=1}^N E_{\mathrm{int},-i}.
\]
Then
$
\mathbb{P}(E_{\mathrm{int}}^c)
\le
N\exp\bigl(-(N-1)\xi^2/2\bigr).
$
Moreover, for all sufficiently large $N$, on the event $E_{\mathrm{int}}$, the augmented threshold rank and leave-one-out threshold ranks satisfy $\hat j<N$ and $\hat j_{-i}<N$ for all $i\in [N]$.
\end{lemma}

The proof of the above result is provided in Section \ref{sec:proof-lem-eint-all-high-prob}.

\subsection{Concomitants of order statistics}
\label{app:concomitants}

We record a standard conditional-independence property used in \Cref{thm:EK-log}.

\begin{lemma}[Conditional independence of concomitants]\label{lem:concomitants}
Let $(Z_i,U_i)_{i=1}^n$ be i.i.d.\ pairs, where $Z_i\in\mathbb R$ has an atomless distribution.
Let $Z_{(1)}<\cdots<Z_{(n)}$ be the order statistics and let $v_1,\dots,v_n$ be the (a.s.\ unique) permutation such that
$Z_{v_j}=Z_{(j)}$.
Define the concomitants $U_{(j)}:=U_{v_j}$ for $j\in [n]$.
Define the sigma-field $\mathcal G:=\sigma(Z_{(1)},\dots,Z_{(n)})$.
Then conditional on $\mathcal G$, the random variables $U_{(1)},\dots,U_{(n)}$ are independent, and
for each $j\in [n]$, the conditional law of $U_{(j)}$ given $\mathcal G$ is the same as the law of $U_1$ given $Z_1=Z_{(j)}$.
\end{lemma}

The proof of the above result is provided in Section \ref{sec:proof-lem-concomitants}.

\subsection{Threshold lemma for the population oracle}
\label{app:tau-pop}

\begin{lemma}\label{lem:tau-pop}
Assume $0\le B<\E[\Delta(X)]$ and let $\mu:=P_X$. Define
\[
F(t):=\E[\Delta(X)\mathbf 1\{\Delta(X)<t\}]+t\P(\Delta(X)=t)
=\int_{\{\Delta\le t\}}\Delta\,d\mu,
\qquad t\in[0,1],
\]
and
\[
\tau:=\inf\{t\in[0,1]:F(t)\ge B\}.
\]
Then $\tau$ is well-defined. Moreover:
\begin{enumerate}
\item $B=0$ if and only if $\tau=0$.
\item If $B>0$, then $\tau\in(0,1]$ and
\[
\E[\Delta(X)\mathbf 1\{\Delta(X)<\tau\}]
\le B
\le \E[\Delta(X)\mathbf 1\{\Delta(X)<\tau\}]+\tau\P(\Delta(X)=\tau).
\]
\end{enumerate}
\end{lemma}

The proof of the above result is provided in Section \ref{sec:proof-lem-tau-pop}.

\section{Proofs}

\subsection{Proof of the population oracle}\label{sec:proof-thm-zero-one-pop}
\begin{proof}
Let $\mu:=P_X$. For any measurable deterministic policy $\pi$, define
$S(\pi):=\{x\in\mathcal X:\pi(x)=\pi_0(x)\}$.
Construct a new policy $\tilde\pi$ by
\[
\tilde\pi(x):=
\begin{cases}
\pi_0(x), & x\in S(\pi),\\
\pi_*(x), & x\notin S(\pi).
\end{cases}
\]
Then pointwise $\mathbf{1}\{\tilde\pi(x)=\pi_0(x)\}\ge \mathbf{1}\{\pi(x)=\pi_0(x)\}$, hence
$\P(\tilde\pi(X)=\pi_0(X))\ge \P(\pi(X)=\pi_0(X))$.
On the other hand, for each $x\notin S(\pi)$ we have $\pi(x)\neq \pi_0(x)$ and by optimality of $\pi_*(x)$,
$g(\pi_*(x),x)=g_*(x)\le g(\pi(x),x)$. Therefore $g(\tilde\pi(x),x)\le g(\pi(x),x)$ for all $x$, and
$\P(\ell(\tilde\pi(X),Y)\ge c)=\E[g(\tilde\pi(X),X)]\le \E[g(\pi(X),X)]=\P(\ell(\pi(X),Y)\ge c)$.
Thus if $\pi$ is feasible, then so is $\tilde\pi$, and $\tilde\pi$ achieves at least as large an objective value. Hence some optimum lies in the two-action class.

For any measurable set $S\subseteq\mathcal X$, define the two-action policy
\[
\pi_S(x):=
\begin{cases}
\pi_0(x), & x\in S,\\
\pi_*(x), & x\notin S.
\end{cases}
\]
Its agreement rate is
\[
\P(\pi_S(X)=\pi_0(X))
=\mu(S)+\mu(S^c\cap\{\Delta=0\})
=\mu(S\cup\{\Delta=0\}),
\]
Indeed, $\Delta(x)=0$ if and only if $g_0(x)=g_*(x)$, which holds if and only if
$\pi_0(x)\in\argmin_{a\in\mathcal A} g(a,x)$; by the tie-breaking convention for
$\pi_*$, this is equivalent to $\pi_*(x)=\pi_0(x)$.
Its risk is
\[
\P(\ell(\pi_S(X),Y)\ge c)
=\E[g(\pi_S(X),X)]
=\int_S g_0\,d\mu+\int_{S^c} g_*\,d\mu
=G_*+\int_S \Delta\,d\mu.
\]
Since $\Delta=0$ on $\{\Delta=0\}$, we have $\int_{S\cup\{\Delta=0\}}\Delta\,d\mu=\int_S\Delta\,d\mu$. Thus, for the set-optimization problem below we may restrict without loss of generality to sets $S$ that contain $\{\Delta=0\}$, in which case $\P(\pi_S(X)=\pi_0(X))=\mu(S)$.
Therefore \Cref{eq:zero-one-obj} is equivalent to the pure set problem
\begin{equation}\label{eq:set-problem-pop}
\max_{S\subseteq\mathcal X\text{ measurable}}\ \mu(S)
\quad\text{s.t.}\quad
\int_S \Delta\,d\mu\le B.
\end{equation}
If $B<0$, \eqref{eq:set-problem-pop} is infeasible. If $B\ge \int_{\mathcal X}\Delta\,d\mu$ (equivalently $G_0\le \varepsilon$), then $S=\mathcal X$ is optimal and corresponds to $\pi_0$.
Henceforth assume $0\le B<\int_{\mathcal X}\Delta\,d\mu$.

If $B=0$, feasibility forces $\Delta=0$ $\mu$-a.e.\ on $S$, so every feasible $S$ satisfies $S\subseteq\{\Delta=0\}$ up to $\mu$-null sets. Thus $S^*:=\{\Delta=0\}$ is feasible and maximizes $\mu(S)$, proving optimality in this case (and in particular $\tau=0$ by \Cref{lem:tau-pop}).

Now assume $B>0$. By \Cref{lem:tau-pop}, we have $\tau\in(0,1]$.

\emph{(a) A universal upper bound.}
Fix any $\lambda\ge 0$. For any feasible $S$ with $\int_S\Delta\,d\mu\le B$,
\[
\mu(S)=\int_S 1\,d\mu
=\int_S(1-\lambda\Delta)\,d\mu+\lambda\int_S\Delta\,d\mu
\le \int_S(1-\lambda\Delta)\,d\mu+\lambda B.
\]
Moreover, for any set $S$,
\[
\int_S(1-\lambda\Delta)\,d\mu
\le \int_S(1-\lambda\Delta)_+\,d\mu
\le \int_{\mathcal X}(1-\lambda\Delta)_+\,d\mu,
\]
and equality is attained by taking $S=\{1-\lambda\Delta>0\}\cup E_0$ for any measurable $E_0\subseteq\{1-\lambda\Delta=0\}$.
Thus every feasible $S$ satisfies
\[
\mu(S)\le J(\lambda):=\int_{\mathcal X}(1-\lambda\Delta)_+\,d\mu+\lambda B,
\]
and hence
\[
\sup\Big\{\mu(S): \int_S\Delta\,d\mu\le B\Big\}\le \inf_{\lambda\ge 0} J(\lambda).
\]

\emph{(b) Choose $\lambda^*=1/\tau$ and construct $S^*$ attaining the bound.}
By \Cref{lem:tau-pop},
\[
\int_{\{\Delta<\tau\}}\Delta\,d\mu\ \le\ B\ \le\ \int_{\{\Delta<\tau\}}\Delta\,d\mu+\tau\,\mu(\Delta=\tau),
\]
which implies $s=\frac{B-\int_{\{\Delta<\tau\}}\Delta\,d\mu}{\tau}\in[0,\mu(\Delta=\tau)]$.
Since $\mu$ is atomless, by Sierpi\'nski's theorem \citep{sierpinski1922fonctions}: for any measurable $A$ and any $u\in[0,\mu(A)]$, there exists a measurable $E\subseteq A$ with $\mu(E)=u$.
Apply this with $A=\{\Delta=\tau\}$ and $u=s$ to obtain $E\subseteq\{\Delta=\tau\}$ with $\mu(E)=s$, and define $S^*:=\{\Delta<\tau\}\cup E$.
Then $\Delta=\tau$ on $E$, so
\[
\int_{S^*}\Delta\,d\mu
=\int_{\{\Delta<\tau\}}\Delta\,d\mu+\int_E \Delta\,d\mu
=\int_{\{\Delta<\tau\}}\Delta\,d\mu+\tau\,\mu(E)
=\int_{\{\Delta<\tau\}}\Delta\,d\mu+\tau s
=B,
\]
so $S^*$ is feasible and the constraint is tight.

Now set $\lambda^*:=1/\tau$. Observe that
\[
(1-\lambda^*\Delta)_+
=\Big(1-\frac{\Delta}{\tau}\Big)_+
=\mathbf{1}\{\Delta<\tau\}\Big(1-\frac{\Delta}{\tau}\Big),
\]
so
\[
\int_{\mathcal X}(1-\lambda^*\Delta)_+\,d\mu
=\mu(\Delta<\tau)-\frac{1}{\tau}\int_{\{\Delta<\tau\}}\Delta\,d\mu.
\]
Therefore
\begin{align*}
J(\lambda^*)
&=\mu(\Delta<\tau)-\frac{1}{\tau}\int_{\{\Delta<\tau\}}\Delta\,d\mu+\frac{B}{\tau} \\
&=\mu(\Delta<\tau)+\frac{B-\int_{\{\Delta<\tau\}}\Delta\,d\mu}{\tau}
=\mu(\Delta<\tau)+s
=\mu(S^*).
\end{align*}
Combining with the bound $\mu(S)\le J(\lambda^*)$ for all feasible $S$ shows that $S^*$ is optimal for \eqref{eq:set-problem-pop}. Translating back, the corresponding deterministic policy $\pi^*:=\pi_{S^*}$ is optimal for \Cref{eq:zero-one-obj}, with $\P\bigl(\ell(\pi^*(X),Y)\ge c\bigr)=\varepsilon$ and $\P(\pi^*(X)=\pi_0(X))=\mu(\Delta<\tau)+\mu(E)=\mu(\Delta<\tau)+s$.
\end{proof}

\subsection{Proof of the rank-stability proposition}\label{sec:proof-prop-stab-k}
\begin{proof}
Given $i\in[N]$, by \eqref{eq:L-nonmono-stab},
\[
L(Z_i;\hat\tau_{-i})-L(Z_i;\hat\tau_{1:N})
=
W_i\Big(\mathbf 1\{\hat\Delta_i<\hat\tau_{-i}\}-\mathbf 1\{\hat\Delta_i<\hat\tau_{1:N}\}\Big).
\]
By the conventions in \Cref{app:formal-rank-stability}, for all $i\in [N]$ we have
\[
\mathbf 1\{\hat\Delta_i<\hat\tau_{-i}\}=\mathbf 1\{q_i<\hat j_{-i}\},
\qquad
\mathbf 1\{\hat\Delta_i<\hat\tau_{1:N}\}=\mathbf 1\{q_i<\hat j\}.
\]
Thus $|L(Z_i;\hat\tau_{-i})-L(Z_i;\hat\tau_{1:N})|$ can be nonzero only if $q_i$ lies between $\hat j$ and $\hat j_{-i}$.
Define the disjoint sets
\[
S_+:=\{i:\hat j_{-i}>\hat j,\ \hat j\le q_i<\hat j_{-i}\},
\qquad
S_-:=\{i:\hat j_{-i}<\hat j,\ \hat j_{-i}\le q_i<\hat j\}.
\]
Since $|W_i|\le1$,
\[
\sum_{i=1}^N |L(Z_i;\hat\tau_{-i})-L(Z_i;\hat\tau_{1:N})|
\le |S_+|+|S_-|.
\]
If $i\in S_+$, then $\hat j_{-i}\le \hat j+K$, hence $q_i\in\{\hat j,\hat j+1,\dots,\hat j+K-1\}$, an interval containing at most $K$ possible positive ranks. Since the ranks are a.s.~distinct by assumption, $|S_+|\le K$ a.s. Similarly, $|S_-|\le K$ a.s. Dividing the resulting bound by $N$ proves the claim.
\end{proof}

\subsection{Proof of the uniform perturbation lemma}\label{sec:proof-lem-t-perturb}
\begin{proof}
By the definition of the leave-one-out partial sum $T_j^{-i}$, we have 
$T_j^{-i}=T_j-(I^s_i-\varepsilon)-W_i\mathbf 1\{q_i<j\}$.
Therefore
$T_j^{-i}-T_j=-(I^s_i-\varepsilon)-W_i\mathbf 1\{q_i<j\}$.
If $q_i<j$, then
$T_j^{-i}-T_j=\varepsilon-I^0_i$; if $q_i\ge j$, then
$T_j^{-i}-T_j=\varepsilon-I^s_i$. In either case the difference is one of
$\varepsilon$ or $-(1-\varepsilon)$, and hence has absolute value at most $1$.
\end{proof}

\subsection{Proof of the one-step bound}\label{sec:proof-lem-one-step}
\begin{proof}
First,
\[
T_1-T_0
=
\sum_{t=1}^N W_{(t)}
\left(\mathbf 1\{t<1\}-\mathbf 1\{t<0\}\right)
=0,
\]
so \(T_1=T_0\).
Similarly, by definition of \(a_i\), there is no index $t$ in $\mathcal J_{-i}$ satisfying \(0<t<a_i\). Hence
\[
T^{-i}_{a_i}-T^{-i}_0
=
\sum_{t\in[N]\setminus\{q_i\}} W_{(t)}
\mathbf 1\{t<a_i\}
=0,
\]
and \(T^{-i}_{a_i}=T^{-i}_0\).

Next,
$T_{N+1}-T_N=W_{(N)}$,
so
$|T_{N+1}-T_N|\le 1$.
Similarly, \(b_i\) is the largest index in $\mathcal J_{-i}$  below \(N+1\), so 
$T^{-i}_{N+1}-T^{-i}_{b_i}=W_{(b_i)}$,
and therefore
$|T^{-i}_{N+1}-T^{-i}_{b_i}|\le 1$.

If we assume \(0<\hat j<N+1\), then \(\hat j\in[N]\), and
$T_{\hat j+1}-T_{\hat j}
=
W_{(\hat j)}$.
Since \(W_{(\hat j)}\in[-1,1]\),
we obtain $|T_{\hat j+1}-T_{\hat j}|\le 1$.

Similarly, For the leave-one-out sample, suppose \(0<\hat j_{-i}<N+1\). By definition of
\(\hat j^+_{-i}\), there is no index $t$ in $\mathcal J_{-i}$ satisfying \(\hat j_{-i} < t < \hat j^+_{-i}\). Therefore
\[
T^{-i}_{\hat j^+_{-i}}-T^{-i}_{\hat j_{-i}}
=
\sum_{t\in[N]\setminus\{q_i\}} W_{(t)}
\left(
\mathbf 1\{t<\hat j^+_{-i}\}
-
\mathbf 1\{t<\hat j_{-i}\}
\right)
\]
contains at most one term. Since each \(W_{(t)}\in[-1,1]\),
$|T^{-i}_{\hat j^+_{-i}}-T^{-i}_{\hat j_{-i}}|\le 1$, completing the proof.
\end{proof}

\subsection{Proof of the low-sum block lemma}\label{sec:proof-lem-block-3}
\begin{proof}
By assumption, $T_{\hat j}\le0<T_{\hat j+1}$.
If $\hat j=0$, then \Cref{lem:endpoint-crossing} gives $T_1=T_0$,
contradicting $T_{\hat j}\le0<T_{\hat j+1}$. Hence $\hat j\in[N]$.
Thus \Cref{lem:endpoint-crossing} implies $|T_{\hat j+1}-T_{\hat j}|\le 1$,
and $T_{\hat j}>-1$.

Given $i\in [N]$, consider the leave-one-out sample with observation $i$ removed. 
By assumption, $T^{-i}_{\hat j_{-i}}\le0<T^{-i}_{\hat j_{-i}^+}$.
Since by assumption $\hat j_{-i}\in [N]$, \Cref{lem:endpoint-crossing} implies 
$T^{-i}_{\hat j_{-i}^+}-T^{-i}_{\hat j_{-i}}\le1$.
Combining these bounds, we deduce $T^{-i}_{\hat j_{-i}}>-1$.
By \Cref{lem:T-perturb},
$T_{\hat j_{-i}}\le T^{-i}_{\hat j_{-i}}+1\le1$ and 
$T_{\hat j_{-i}}>T^{-i}_{\hat j_{-i}}-1>-2$.

If $\hat j_{-i}\ge \hat j$, then
\[
\sum_{t=\max\{1,\hat j\}}^{\hat j_{-i}-1} W_{(t)} = T_{\hat j_{-i}}-T_{\hat j}<2.
\]
If $\hat j_{-i}<\hat j$, then
\[
\sum_{t=\max\{1,\hat j_{-i}\}}^{\hat j-1} W_{(t)} = T_{\hat j}-T_{\hat j_{-i}}<2.
\]
In either case, the sum is an integer, since each $W_{(t)}\in\{-1,0,1\}$, and is at most unity.
\end{proof}

\subsection{Proof of the concomitant expectation bound}\label{sec:proof-thm-ek-log}
\begin{proof}
First, let $p_s:=\mathbb{P}(I^s_1=1)$ and $p_0:=\mathbb{P}(I^0_1=1)$.
Let $\kappa > 0$ and $\xi > 0$ satisfy $p_s \le \varepsilon - \kappa$ and $p_0 \ge \varepsilon + \xi$.
Let $F$ be the cumulative distribution function of $\hat\Delta(X)$. 
Set $c_0:=\kappa/4$ and $i_0:=\lceil c_0N\rceil$.
Since by assumption, $\hat \Delta(X)$ has an atomless distribution on $[0,1]$, we may select $\zeta>0$ such that $\P(\hat\Delta(X)<\zeta)\le c_0/2$. 
Define $\lambda_W:=c_W \zeta^{\beta}>0$.

Define the suffix low-sum block length
\[
M^{(1)}_{i_0}(N)
:=
\max\Bigl(\{0\}\cup\Bigl\{h\in[N] :\ \exists\, r\in\{i_0,\dots,N-h+1\}\ \text{s.t.}
\sum_{t=r}^{r+h-1} W_{(t)} \le 1
\Bigr\}\Bigr).
\]

Let
$S_0:=\sum_{i=1}^N (I^s_i-\varepsilon)$ and
$E_{\mathrm{start}}:=\{S_0\le -\kappa N/2\}$.
For $i\in [N]$, define
\[
\hat p_{0,-i}:=(N-1)^{-1}\sum_{k\in [N]\setminus \{i\}} I^0_k,
\] and define the events \[
E_{\mathrm{int},-i}:=\Big\{\hat p_{0,-i}>\varepsilon+\frac{\xi}{2}\Big\},
\qquad
E_{\mathrm{int}}:=\bigcap_{i=1}^N E_{\mathrm{int},-i}.
\]

On \(E_{\mathrm{start}}\), for all sufficiently large \(N\),
\[
T_0=S_0+1-\varepsilon
\le -\frac{\kappa N}{2}+1-\varepsilon
\le0.
\]
Thus index \(0\) is \(\hat R^+_{1:N}\)-feasible. By
\Cref{lem:endpoint-crossing}, \(T_1=T_0\), so index \(1\) is
\(\hat R^+_{1:N}\)-feasible, and $\hat j\ge1.$

Given \(i\in[N]\), we have
\[
T^{-i}_0
=
T_0-(I^s_i-\varepsilon)
=
S_0+1-I^s_i
\le S_0+1
\le -\frac{\kappa N}{2}+1
\le0
\]
for all sufficiently large \(N\). Hence index \(0\) is
\(\hat R^+_{-i}\)-feasible. Let
\[
a_i:=\min\bigl\{\mathcal J_{-i}\setminus\{0,N+1\}\bigr\}.
\]
By \Cref{lem:endpoint-crossing}, \(T^{-i}_{a_i}=T^{-i}_0\), so \(a_i\) is also
\(\hat R^+_{-i}\)-feasible. Therefore $\hat j_{-i}\ge a_i\ge1.$

By \Cref{lem:Eint-all-high-prob}, on $E_{\mathrm{int}}$ we have $\hat j<N$ and $\hat j_{-i}<N$ for every $i\in[N]$. 
Thus, on
$E_{\mathrm{start}}\cap E_{\mathrm{int}}$, 
we have $T_{\hat j}\le 0 < T_{\hat j+1}$
and
$T_{\hat j_{-i}}^{-i}\le 0 < T_{\hat j_{-i}^+}^{-i}$
for $i\in [N]$,
where $\hat j_{-i}^+$ is defined in \Cref{eq:j-minus-i-plus}.

Next, for any $j\ge1$, 
\begin{align*}
T_j = S_0+\sum_{t<j}W_{(t)}+(1-\varepsilon)
\le S_0+(j-1)+(1-\varepsilon),
\end{align*}
because $W_{(t)}\le1$. Therefore, on $E_{\mathrm{start}}$, if $j\ge1$ and $T_j>-2$, then
$-2<S_0+(j-1)+(1-\varepsilon)$, so
$j>\kappa N/2-2+\varepsilon$. 
On $E_{\mathrm{start}}\cap E_{\mathrm{int}}$,
since $0 < \hat j < N+1$ and $T_{\hat j+1}>0$, 
\Cref{lem:endpoint-crossing} implies
$T_{\hat j}>-1$.
Similarly, for $i\in [N]$, 
since $0 < \hat j_{-i} < N+1$ and $T_{\hat j_{-i}^+}^{-i}>0$, 
\Cref{lem:endpoint-crossing} implies $T_{\hat j_{-i}}^{-i}>-1$.
By \Cref{lem:T-perturb}, we deduce $T_{\hat j_{-i}} > -2$.
Hence, for all sufficiently large $N$,
$\min\{\hat j,\hat j_{-i}\}\ge i_0$ for all $i\in [N]$ on $E_{\mathrm{start}}\cap E_{\mathrm{int}}$.

Now fix \(i\in[N]\) and set
\(r_i:=\min\{\hat j,\hat j_{-i}\}\) and 
\(h_i:=|\hat j-\hat j_{-i}|\).
Since \(r_i\ge i_0\), \Cref{lem:block-3} gives
\[
\sum_{t=r_i}^{r_i+h_i-1} W_{(t)}\le1,
\]
Thus, if \(h_i>0\),
the interval \(\{r_i,\ldots,r_i+h_i-1\}\) is a contiguous block of indices contained in the suffix
\(\{i_0,\ldots,N\}\), with cumulative drift at most \(1\). By the definition of
\(M^{(1)}_{i_0}(N)\), this implies \(h_i\le M^{(1)}_{i_0}(N)\). Taking the maximum
over \(i\in[N]\), we obtain
\[
K\mathbf 1_{E_{\mathrm{start}}\cap E_{\mathrm{int}}}
\le M^{(1)}_{i_0}(N).
\]

Since $K\le N+1$ always, this implies the bound
\begin{equation}\label{eq:EK-split}
\E[K]
\le
\E\big[M^{(1)}_{i_0}(N)\big]
+(N+1)\P(E_{\mathrm{start}}^c)
+(N+1)\P(E_{\mathrm{int}}^c).
\end{equation}

We control each term in turn. Since $p_s \le \ep - \kappa$ by assumption, we have $\E[S_0]=N(p_s-\varepsilon)\le-\kappa N$. Since $|I^s_i-\varepsilon|\le 1$, Hoeffding's inequality \citep{hoeffding1963probability} gives
\[
\P(E_{\mathrm{start}}^c)
=\P\Big(S_0-\E[S_0]\ge \frac{\kappa N}{2}\Big)
\le e^{-\kappa^2N/2}.
\]
By \Cref{lem:Eint-all-high-prob},
$\P(E_{\mathrm{int}}^c)\le N e^{-(N-1)\xi^2/2}$.

Thus, the second and third terms of \Cref{eq:EK-split} are exponentially small in $N$, and it remains to bound $\E[M^{(1)}_{i_0}(N)]$. Define the sigma-field
$\mathcal G:=\sigma(\hat\Delta_{(1)},\dots,\hat\Delta_{(N)})$ and the event
$E_{\mathrm{quant}}:=\{\hat\Delta_{(i_0)}\ge \zeta\}$.
Note that $E_{\mathrm{quant}}\in \mathcal G$.
On $E_{\mathrm{quant}}$, $\hat\Delta_{(t)}\ge \zeta$ for all $t\ge i_0$. By the drift lower bound assumption, we have
$\E[W_{(t)}\mid\mathcal G]=\mu_W(\hat\Delta_{(t)})\ge c_W \zeta^{\beta}=\lambda_W$ for all $t\ge i_0$, a.s.\ on $E_{\mathrm{quant}}$.
By \Cref{lem:concomitants}, conditional on $\mathcal G$, the variables $W_{(1)},\dots,W_{(N)}$ are independent, and $W_{(t)}$ has the conditional law of $W$ given $\hat\Delta(X)=\hat\Delta_{(t)}$. Applying \Cref{cor:low-sum-b} conditionally on $\mathcal G$ with $b=1$ gives, on $E_{\mathrm{quant}}$,
\[
\E\big[M^{(1)}_{i_0}(N)\mid\mathcal G\big]
\le \frac{16}{\lambda_W^2} \log N + \frac{2}{\lambda_W} + 2.
\] Since $M^{(1)}_{i_0}(N)\le N$, we have the bound
\[ \E\big[M^{(1)}_{i_0}(N)] \le  \left( \frac{16}{\lambda_W^2} \log N + \frac{2}{\lambda_W} + 2 \right)
+ N \mathbb{P}(E_{\mathrm{quant}}^c), \]
and it suffices to control $\mathbb{P}(E_{\mathrm{quant}}^c)$. Let $B_\zeta:=|\{i\in[N]:\hat\Delta_i<\zeta\}|$. By the choice of $\zeta$, $B_\zeta$ is stochastically dominated by a $\mathrm{Binom}(N,c_0/2)$ random variable. Since $i_0=\lceil c_0N\rceil$, the event $E_{\mathrm{quant}}^c$ implies $B_\zeta\ge i_0\ge c_0N$. Applying a Chernoff bound \citep{chernoff1952measure} to the binomial variable gives
$\P(E_{\mathrm{quant}}^c)\le e^{-c_0N/6}$.
Thus $\E[M^{(1)}_{i_0}(N)] \le  \left( \frac{16}{\lambda_W^2} \log N + \frac{2}{\lambda_W} + 2 \right)
+ o(1)$, so that for sufficiently large $n$, by the definition of $C_1, C_2$ in \Cref{app:constants}, we have 
\[ \E[K] \le \frac{16}{\lambda_W^2} \log N + \frac{2}{\lambda_W} + 3 =: C_1\log(n+1) + C_2.\]
\end{proof}

\subsection{Proof of expected risk control from rank stability}\label{sec:proof-cor-beta-logn-over-n}
\begin{proof}
Exchangeability gives
\[
\E\big[L(Z_N;\hat\tau_{-N})\big]
=\E\left[\frac1N\sum_{i=1}^N L(Z_i;\hat\tau_{-i})\right].
\]
By \Cref{prop:stab-K},
\[
\frac1N\sum_{i=1}^N L(Z_i;\hat\tau_{-i})\le
\frac1N\sum_{i=1}^N L(Z_i;\hat\tau_{1:N})+\frac{2K}{N}.
\]
On the event $E_{\mathrm{start}}$ from the proof of \Cref{thm:EK-log}, the threshold $\tau=0$ is $\hat R_{1:N}^+$-feasible for all sufficiently large $N$, so that $\hat R_{1:N}^+(\hat \tau_{1:N})\le \ep$. Therefore $\sum_{i=1}^N L(Z_i;\hat\tau_{1:N})+1\le (N+1)\varepsilon$, and, because $\varepsilon\le1$,
\[
\frac1N\sum_{i=1}^N L(Z_i;\hat\tau_{1:N})\le\varepsilon
\]
on $E_{\mathrm{start}}$. On $E_{\mathrm{start}}^c$, the average loss is at most one. Therefore
\[
\E\big[L(Z_N;\hat\tau_{-N})\big]
\le \varepsilon+\P(E_{\mathrm{start}}^c)+\frac{2\E[K]}{N}.
\]
The proof of \Cref{thm:EK-log} gives $\P(E_{\mathrm{start}}^c)\le e^{-\kappa^2N/2}$, and \Cref{thm:EK-log} also gives $\E[K]\le C_1\log N + C_2$ for sufficiently large $n$. Substituting $N=n+1$ and using the definition $C_3 := 3C_1$ from \Cref{app:constants} completes the proof.
\end{proof}

\subsection{Proof of the exact-safe fallback theorem}\label{sec:proof-thm-safe-policy}
\begin{proof}
Exact-safety gives, for every \(z=(x,y)\in \mathcal X\times \mathcal Y\) and every
\(\tau\in[0,1]\cup\{\top\}\),
\[
L(z;\tau)=\mathbf 1\{\ell(\pi_0(x),y)\ge c\}\,
\mathbf 1\{\hat\Delta(x)<\tau\}.
\]
Thus \(L(z;\tau)\) is non-decreasing in the extended threshold order.

As in \Cref{app:formal-rank-stability}, let \(N:=n+1\), and consider an augmented sample $Z_{1:N}$. 
For each \(i\in[N]\), let \(\hat\tau_{-i}\) be the threshold
selected by \Cref{alg:crc-nonmono-conservative} with
\(\hat \pi_* = \pi_{\mathrm{safe}}\) from the leave-one-out sample
\(Z_{[N]\setminus\{i\}}\). By exchangeability,
\[
\E\big[L(Z_N;\hat\tau_{-N})\big]
=
\E\left[\frac1N\sum_{i=1}^N L(Z_i;\hat\tau_{-i})\right].
\]
Next, define the (un-bumped) empirical risk over the augmented sample as
\[ \hat R_{1:N}(\tau) := \frac1N\sum_{j=1}^N L(Z_j;\tau) \]
for $\tau\in [0,1]\cup \{\top\}$.
Let
\(\mathcal T_{1:N}:=\{0,\top\}\cup\{\hat\Delta(X_i):i\in[N]\}\), and define the modified threshold
\begin{align}\label{eq:unbumped-thresh}
\hat\tau_{1:N}^{\circ}
:=
\max\left\{\tau\in\mathcal T_{1:N}:
\hat R_{1:N}(\tau) \le\varepsilon
\right\},
\end{align}
where the maximum is taken in the extended order.
We say a threshold $\tau\in [0,1]\cup \{\top\}$ is \textit{$\hat R_{1:N}$-feasible} if $\hat R_{1:N}(\tau)\le \varepsilon$. 
Note that \(\tau=0\) is $\hat R_{1:N}$-feasible because \(L(z;0)\equiv 0\).

Fix \(i\in[N]\). If the set of $\hat R_{-i}^+$-feasible indices is non-empty, then since \(\hat\tau_{-i}\) is $\hat R_{-i}^+$-feasible, we deduce
\[
\sum_{j\in [N]\setminus \{i\}} L(Z_j;\hat\tau_{-i})+1\le N\varepsilon.
\]
Since \(L(Z_j;\hat\tau_{-i})\le1\), it follows that
\(\sum_{j=1}^N L(Z_j;\hat\tau_{-i})\le N\varepsilon\). 
On the other hand, if the set of $\hat R_{-i}^+$-feasible indices is empty, then \(\hat\tau_{-i}=0\) by convention and the same
inequality holds because \(L(z;0)\equiv 0\). Therefore each
\(\hat\tau_{-i}\) is $\hat R_{1:N}$-feasible.
Consequently, by the monotonicity of $\hat R_{1:N}(\tau)$ in $\tau$, we have
\(\hat\tau_{-i}\le\hat\tau_{1:N}^{\circ}\) for every \(i\in [N]\). As a result,
\[
\frac1N\sum_{i=1}^N L(Z_i;\hat\tau_{-i})
\le
\frac1N\sum_{i=1}^N L(Z_i;\hat\tau_{1:N}^{\circ})
\le\varepsilon.
\]
Taking expectations proves
\(\E[L(Z_N;\hat\tau_{-N})]\le\varepsilon\). Since \(\hat\tau_{-N}\) is the
threshold selected from the calibration observations \(Z_{[n]}\), this implies the result.
\end{proof}

\subsection{Proof of the near-optimality theorem}\label{sec:proof-thm-near-opt}
\begin{proof}
By the exact-safety assumption, we have
\[
L(X,Y;\tau)=\mathbf 1\{\ell(\pi_0(X),Y)\ge c\}\,\mathbf 1\{\hat\Delta(X)<\tau\}
\]
a.s.\ for $\tau\in[0,1]$. Let
$\hat R(\tau)=n^{-1}\sum_{i=1}^nL(X_i,Y_i;\tau)$ and
$R(\tau)=\E[L(X_i,Y_i;\tau)]$.
Given $\tau\in [0,1]\cup \{\top\}$, we say that $\tau$ is \textit{$\hat R^+$-feasible} if $\hat R^+(\tau)\le \ep$.
By \Cref{cor:step1-eps1}, with probability at least $1-\delta$,
\[
\sup_{\tau\in[0,1]}|\hat R(\tau)-R(\tau)|\le\varepsilon_{1,n}(\delta).
\]
We work on this event.

The oracle safe-action loss is
$L^*(X,Y;\tau)=\mathbf 1\{\ell(\pi_0(X),Y)\ge c\}\mathbf 1\{\Delta(X)<\tau\}$.
For $\tau\in[0,1]$, set
\[
E_\tau:=\{\mathbf 1\{\hat\Delta(X)<\tau\}\ne\mathbf 1\{\Delta(X)<\tau\}\}.
\]
Since the two losses can differ only on $E_\tau$,
\[
|R(\tau)-R^*(\tau)|\le \P(E_\tau).
\]
For any $u>0$, if $E_\tau$ holds and $|\hat\Delta(X)-\Delta(X)|\le u$, then
$|\Delta(X)-\tau|\le u$. Hence
\[
E_\tau\subseteq
\{|\hat\Delta(X)-\Delta(X)|>u\}\cup\{|\Delta(X)-\tau|\le u\}.
\]
Taking the supremum over $\tau\in[0,1]$ and then the infimum over $u>0$ gives
\[
\sup_{\tau\in[0,1]}|R(\tau)-R^*(\tau)|\le \varepsilon_\Delta.
\]
The same containment, applied directly to the two threshold indicators, gives
\[
\sup_{\tau\in[0,1]}
\big|\P(\hat\Delta(X)<\tau)-\P(\Delta(X)<\tau)\big|
\le \varepsilon_\Delta.
\]
Because $\hat R^+(\tau)=\frac{n}{n+1}\hat R(\tau)+\frac1{n+1}$ and $R^*(\tau)\in[0,1]$, on the same event
\[
\sup_{\tau\in[0,1]}|\hat R^+(\tau)-R^*(\tau)|\le
\frac{n}{n+1}\varepsilon_{1,n}(\delta)+\frac1{n+1}+\varepsilon_\Delta
=\varepsilon_n(\delta).
\]

Let $\nu=\nu_n(\delta)$. By assumption, $\nu\le\rho$, so both $\tau^*-\nu$ and $\tau^*+\nu$ lie in $[0,1]$ and the local margin condition applies to them. Since $R^*(\tau^*)=\varepsilon$ and $R^*$ is non-decreasing,
\[
R^*(\tau^*-\nu)\le \varepsilon-C_R\nu,
\qquad
R^*(\tau^*+\nu)\ge \varepsilon+C_R\nu.
\]
The definition of $\nu$ gives $C_R\nu=\varepsilon_n(\delta)+1/(n+1)$. Hence
\[
\hat R^+(\tau^*-\nu)\le\varepsilon-\frac1{n+1}<\varepsilon,
\qquad
\hat R^+(\tau^*+\nu)\ge\varepsilon+\frac1{n+1}>\varepsilon.
\]
Note that $\hat R^+$ is non-decreasing in the extended threshold order. The algorithm selects over the grid $\mathcal T_n=\{0,\top\}\cup\{\hat\Delta(X_i):i\in[n]\}$, while the two thresholds $\tau^*-\nu$ and $\tau^*+\nu$ need not themselves be grid points. Since $\tau^*-\nu$ is $\hat R^+$-feasible, the smallest grid point $t^+\in\mathcal T_n$ satisfying $t^+\ge\tau^*-\nu$ has the same empirical loss as $\tau^*-\nu$ unless it is $\top$, in which case all calibration scores are already below $\tau^*-\nu$ and the same conclusion holds. Hence $t^+$ is grid-feasible, so the selected threshold satisfies $\hat\tau\ge t^+\ge\tau^*-\nu$. Since $\tau^*+\nu$ is not $\hat R^+$-feasible, monotonicity makes every grid threshold to its right (including $\top$) infeasible; hence $\hat\tau\le\tau^*+\nu$. In particular, on this event $\hat\tau\in[0,1]$, $\tau^*-\nu\le\hat\tau\le\tau^*+\nu$, and
$|\hat\tau-\tau^*|\le\nu_n(\delta)$.

It remains to translate threshold error into an agreement shortfall. Write $\mu:=P_X$. In the exact-safe setting, any feasible policy $\pi$ with agreement set $A_\pi:=\{x:\pi(x)=\pi_0(x)\}$ satisfies
\[
\int_{A_\pi}\Delta\,d\mu
=
\int_{A_\pi}g_0\,d\mu
\le
\P(\ell(\pi(X),Y)\ge c)
\le\varepsilon.
\]
Thus the agreement of any feasible policy is bounded by the value of the set problem
$\max\{\mu(S):\int_S\Delta\,d\mu\le\varepsilon\}$.
The set $S^*=\{\Delta<\tau^*\}$ is feasible and tight because exact-safety gives
$R^*(\tau^*)=\int_{S^*}\Delta\,d\mu=\varepsilon$, and the policy that keeps $\pi_0$ on $S^*$ and uses $\pi_{\mathrm{safe}}$ elsewhere attains risk $\varepsilon$ and agreement at least $\mu(S^*)$. For any feasible $S$,
\[
\mu(S)
\le
\int_S\left(1-\frac{\Delta}{\tau^*}\right)d\mu+\frac{\varepsilon}{\tau^*}
\le
\int_{\mathcal X}\left(1-\frac{\Delta}{\tau^*}\right)_+d\mu
+\frac{\varepsilon}{\tau^*}
=\mu(\Delta<\tau^*)
+\frac{\varepsilon-R^*(\tau^*)}{\tau^*}
=J^*,
\]
where the equality uses $\tau^*>0$ and
$\left(1-\Delta/\tau^*\right)_+=\mathbf 1\{\Delta<\tau^*\}(1-\Delta/\tau^*)$. Thus $S^*$ is optimal, and the optimal population agreement is
$J^*=\mu(S^*)=\P(\Delta(X)<\tau^*)$. The deployed policy agrees with the baseline at least on the event
$\{\hat\Delta(X)<\hat\tau\}$, so
\begin{align*}
J^*-\P(\hat\pi(X;\hat\tau)=\pi_0(X))
&\le \P(\Delta(X)<\tau^*)-\P(\hat\Delta(X)<\hat\tau)\\
&\le \big|\P(\Delta(X)<\tau^*)-\P(\Delta(X)<\hat\tau)\big|\\
&\quad+\big|\P(\Delta(X)<\hat\tau)-\P(\hat\Delta(X)<\hat\tau)\big|\\
&\le L_J|\hat\tau-\tau^*|+\varepsilon_\Delta\\
&\le L_J\nu_n(\delta)+\varepsilon_\Delta.
\end{align*}
Since the right-hand side is non-negative, the same inequality implies the asserted positive-part shortfall bound.
Unpacking the definition of $\nu_n(\delta)$ and bounding $\hat p \le 1$, we obtain 
\begin{align*}
\left(
J^*-\mathbb P\bigl(\hat\pi(X;\hat\tau)=\pi_0(X)\bigr)
\right)_+
&\le
\frac{L_J}{C_R}
\left[
\sqrt{\frac{\log(4/\delta)}{2n}}
+
\sqrt{\frac{\hat p\log(4/\delta)}{2n}}
+
\frac{2}{n+1}
\right]
+
\left(1+\frac{L_J}{C_R}\right)\varepsilon_\Delta \\
&\le C_4\sqrt{\frac{\log(4/\delta)}{n}}
+
\frac{C_5}{n+1}
+
C_6\varepsilon_\Delta,
\end{align*}
as desired.
\end{proof}

\subsection{Proof of the exact-safe fallback concentration corollary}\label{sec:proof-cor-step1-eps1}
\begin{proof}
If $p=0$, then $W_i=0$ a.s.\ for each $i\in [n]$ and result follows. Assume $p>0$.
Define the \textit{left conditional distribution function}
$F(t):=\P(Z_1<t\mid W_1=1)$ for $t\in [0,1]$.
On $\{m\ge1\}$, define
$\hat F_m(t):= \frac{1}{m}\sum_{i=1}^n W_i\mathbf 1\{Z_i<t\}$; on $\{m=0\}$, set $\hat F_m(t):=0$.
The population and empirical risks satisfy
$R(\tau)=pF(\tau)$ and $\hat R(\tau)=\hat p\,\hat F_m(\tau)$
for all $\tau\in[0,1]$. Hence
\[
\sup_{\tau\in[0,1]}|\hat R(\tau)-R(\tau)|
\le
|\hat p-p|+\hat p\sup_{t\in\R}|\hat F_m(t)-F(t)|.
\]
By Hoeffding's inequality \citep{hoeffding1963probability}, with probability at least $1-\delta/2$,
\[
|\hat p-p|\le \sqrt{\frac{\log(4/\delta)}{2n}}.
\]
Conditional on the index set $S:=\{i\in [n] : W_i=1\}$ and conditional on $m=|S|\ge1$, the selected scores $(Z_i)_{i\in S}$ are i.i.d.~from the conditional law of $Z_1$ given $W_1=1$. The standard Dvoretzky--Kiefer--Wolfowitz inequality for half-lines $(-\infty,t]$ \citep{massart1990tight} also controls strict half-lines $(-\infty,t)$ by taking left limits. Therefore,
\[
\P\left(
\sup_{t\in\R}|\hat F_m(t)-F(t)|
>
\sqrt{\frac{\log(4/\delta)}{2m}}
\;\middle|\; S
\right)\le \frac{\delta}{2}
\]
for every such $S$. On $\{m=0\}$, the product
$\hat p\sup_{t\in\R}|\hat F_m(t)-F(t)|$ is zero by definition. Hence, with probability at least $1-\delta/2$,
\[
\hat p\sup_{t\in\R}|\hat F_m(t)-F(t)|
\le
\sqrt{\frac{\hat p\log(4/\delta)}{2n}}.
\]
A union bound completes the proof.
\end{proof}

\subsection{Proof of the suffix low-sum block bound}\label{sec:proof-cor-low-sum-b}
\begin{proof}
Fix $h\in\{1,\dots,n\}$ and a start index $r\in\{i_0,\dots,n-h+1\}$. Let
$Y_{r,h}:=\sum_{t=r}^{r+h-1}S_t$.
For any $\theta>0$, Markov's inequality gives
\(
\P(Y_{r,h}\le b)
\le e^{\theta b}\prod_{t=r}^{r+h-1}\E[e^{-\theta S_t}].
\)
Since $S_t\in[-1,1]$, Hoeffding's lemma \citep{hoeffding1963probability} yields
\[
\E[e^{-\theta S_t}]
\le \exp\!\left(\frac{\theta^2}{2}-\theta\E[S_t]\right)
\le \exp\!\left(\frac{\theta^2}{2}-\theta \lambda\right),
\]
so
$\P(Y_{r,h}\le b)\le\exp(\theta b+h\theta^2/2-\theta \lambda h)$.
Optimizing over $\theta>0$ gives
\[ \P(Y_{r,h}\le b)\le\exp\left(-\frac{(\lambda h-b)_+)^2}{2h}\right). \]
A union bound over the possible start indices $r$ and block lengths $m$ gives
\[
\P\big(M^{(b)}_{i_0}(n)\ge h\big)
\le \sum_{m=h}^n n\exp\!\left(-\frac{(\lambda m-b)_+^2}{2m}\right).
\]
Since $M^{(b)}_{i_0}(n)$ is non-negative and integer-valued, we may write
$$\E[M^{(b)}_{i_0}(n)]=\sum_{h=1}^n\P(M^{(b)}_{i_0}(n)\ge h)$$ and split the sum at
$h_0:=\left\lceil 16\lambda^{-2}\log n+2b/\lambda\right\rceil$.
The contribution up to and including $h=h_0$ is at most $h_0$. 
For $h > h_0$, we have $h \ge 2b/\lambda$, hence  $\lambda h-b\ge \lambda h/2$.
Thus, the tail sum can be bounded as
\begin{align*}
\sum_{h=h_0+1}^n
\P\!\left(M^{(b)}_{i_0}(n)\ge h\right) 
&\le
\sum_{h=h_0+1}^n
\sum_{m=h}^n ne^{-\lambda^2 m/8} \\ 
&=
n\sum_{m=h_0+1}^n
(m-h_0)e^{-\lambda^2 m/8} 
\le
\frac{n e^{-\lambda^2(h_0+1)/8}}{(1-e^{-\lambda^2/8})^2},
\end{align*}
which by
$n\exp(-\lambda^2h_0/8)\le {1}/{n}$ is bounded by one for sufficiently large $n$. Therefore
$\E[M^{(b)}_{i_0}(n)]\le h_0+1\le 16\frac{1}{\lambda^2}\log n+\frac{2b}{\lambda}+2$ for all sufficiently large $n$.
\end{proof}

\subsection{Proof of the right-endpoint exclusion lemma}\label{sec:proof-lem-eint-all-high-prob}
\begin{proof}
For each fixed $i\in[N]$, $\mathbb E[\hat p_{0,-i}]=p_0\ge\varepsilon+\xi$, and Hoeffding's inequality \citep{hoeffding1963probability} gives
$\P(E_{\mathrm{int},-i}^c)
\le e^{-(N-1)\xi^2/2}.$
A union bound over $i\in [N]$ yields
$\P(E_{\mathrm{int}}^c)
\le N e^{-(N-1)\xi^2/2}.$
On $E_{\mathrm{int}}$, averaging the inequalities
$\sum_{k\in [N]\setminus \{i\}}I^0_k>(N-1)(\varepsilon+\xi/2)$ over $i$ gives
$\sum_{k=1}^N I^0_k>N(\varepsilon+\xi/2)$. 

Therefore
\[
\begin{aligned}
T_{N+1}
&=
\sum_{t=1}^N(I^s_{(t)}-\varepsilon)
+
\sum_{t=1}^N W_{(t)}
+
1-\varepsilon 
&=
\sum_{k=1}^N(I^0_k-\varepsilon)+1-\varepsilon 
&>
\frac{N\xi}{2}+1-\varepsilon.
\end{aligned}
\]
For all sufficiently large \(N\), this is larger than \(1\). Hence
$T_{N+1}>1.$
By \Cref{lem:endpoint-crossing},
$T_N\ge T_{N+1}-1>0.$
Thus both indices \(N\) and \(N+1\) are \(\hat R^+_{1:N}\)-infeasible, and consequently
$\hat j<N.$

Given $i\in [N]$, a similar argument shows that $\hat j_{-i} < N$. 
Indeed, on \(E_{\mathrm{int},-i}\),
\[
\begin{aligned}
T^{-i}_{N+1}
&=
\sum_{t\in[N]\setminus\{q_i\}}(I^s_{(t)}-\varepsilon)
+
\sum_{t\in[N]\setminus\{q_i\}}W_{(t)}
+
1-\varepsilon \\
&=
\sum_{k\in[N]\setminus\{i\}}(I^0_k-\varepsilon)+1-\varepsilon 
>
\frac{(N-1)\xi}{2}+1-\varepsilon.
\end{aligned}
\]
For all sufficiently large \(N\), this is larger than \(1\), so
$T^{-i}_{N+1}>1.$
Let
\[
b_i:=\max\bigl\{\mathcal J_{-i}\setminus\{N+1\}\bigr\}.
\]
By \Cref{lem:endpoint-crossing},
$T^{-i}_{b_i}\ge T^{-i}_{N+1}-1>0$.
Thus both indices \(b_i\) and \(N+1\) are not \(\hat R^+_{-i}\)-feasible. Since no indices in $\mathcal J_{-i}$ lie strictly between \(b_i\) and \(N+1\), no index in $\mathcal J_{-i}$ at or above \(b_i\) is \(\hat R^+_{-i}\)-feasible. Therefore
$\hat j_{-i}<b_i\le N,$
hence $\hat j_{-i}<N$ for each $i\in [N]$.
The result follows.
\end{proof}

\subsection{Proof of the concomitants lemma}\label{sec:proof-lem-concomitants}
\begin{proof}
Let $Q_z$ be a regular conditional law of $U_1$ given $Z_1=z$. Conditional on the full vector $(Z_1,\dots,Z_n)$, the variables $U_1,\dots,U_n$ are independent with conditional laws $Q_{Z_1},\dots,Q_{Z_n}$. The ranking permutation $v_1,\dots,v_n$ is a measurable function of $(Z_1,\dots,Z_n)$ because ties have probability zero. Therefore, conditional on the full vector, $U_{(1)},\dots,U_{(n)}$ are independent with laws $Q_{Z_{(1)}},\dots,Q_{Z_{(n)}}$. Since this product law depends on the $Z$-vector only through the order statistics, conditioning further only on $\mathcal G=\sigma(Z_{(1)},\dots,Z_{(n)})$ gives the same product law. This proves both the conditional independence and the stated conditional marginal distributions.
\end{proof}

\subsection{Proof of the threshold lemma for the population oracle}\label{sec:proof-lem-tau-pop}
\begin{proof}
Since $\Delta\ge 0$ and $\Delta$ is integrable, the sets $\{\Delta\le t\}$ are increasing in $t$, so $F$ is non-decreasing and $F(0)=0$.
Moreover, if $s_n\downarrow t$, then $\Delta\mathbf{1}\{\Delta\le s_n\}\to \Delta\mathbf{1}\{\Delta\le t\}$ pointwise and is dominated by $\Delta$, so $F$ is right-continuous by dominated convergence.
Since $g_0,g_*\in[0,1]$, we have $0\le \Delta\le 1$, and therefore
\[
F(1)=\int_{\X}\Delta\,d\mu=\E[\Delta(X)]>B.
\]
Hence the set $\{t\in[0,1]:F(t)\ge B\}$ is non-empty and $\tau\le 1$.

If $B=0$, then $0\in\{t:F(t)\ge B\}$, so $\tau=0$. Conversely, if $\tau=0$, then by the definition of the infimum there exists a sequence $u_n\downarrow 0$ with $F(u_n)\ge B$ for all $n$. By right-continuity,
\[
F(0)=\lim_{n\to\infty}F(u_n)\ge B.
\]
Since $F(0)=0$ and $B\ge 0$, it follows that $B=0$.
If $B>0$, then $F(0)=0<B$, so $\tau>0$; finiteness has been shown above.

Finally, assume $B>0$ so that $\tau>0$, and take any sequence $t_n\uparrow \tau$ with $0\le t_n<\tau$ for all $n$ (e.g.\ $t_n:=\max\{0,\tau-1/n\}$). Then $\mathbf{1}\{\Delta\le t_n\}\uparrow \mathbf{1}\{\Delta<\tau\}$, so
\[
F(\tau-):=\lim_{n\to\infty}F(t_n)=\int_{\{\Delta<\tau\}}\Delta\,d\mu.
\]
By definition of $\tau$, we have $F(t)<B$ for all $t<\tau$, hence $F(\tau-)\le B$. On the other hand, by definition of the infimum there exists a sequence $u_n\downarrow\tau$ with $F(u_n)\ge B$ for all $n$, so right-continuity yields $F(\tau)=\lim_{n\to\infty}F(u_n)\ge B$. Finally, $F(\tau)=F(\tau-)+\tau\,\mu(\Delta=\tau)$, yielding the stated inequality.
\end{proof}

\end{document}